\documentclass[10pt,onecolumn]{article}
\usepackage{natbib}
\usepackage{our-style}
\usepackage{adjustbox}

\usepackage{hyperref}       
\hypersetup{
    colorlinks=true,
    citecolor=blue
}
\usepackage{url}            
\usepackage{booktabs}       
\usepackage{amsfonts}       
\usepackage{nicefrac}       
\usepackage{microtype}      
\usepackage{xcolor}         
\usepackage{authblk}

\title{\huge Concept Modulation Models: A Unified Framework for Identifiability and Extrapolation}

\author[1]{Soheun Yi\thanks{Correspondence: Soheun Yi, \texttt{soheuny@andrew.cmu.edu}.}}
\author[2]{Yizhou Lu}
\author[2]{Chandler Squires}
\author[2]{Pradeep Ravikumar}
\affil[1]{Department of Statistics and Data Science, Carnegie Mellon University}
\affil[2]{Machine Learning Department, Carnegie Mellon University}
\date{}

\begin{document}

\maketitle

\begin{abstract}
Reliable generalization in conditional latent variable models requires understanding both identifiability and extrapolation: how observed variation across attributes determines latent structure, and how that structure determines distributions at unseen attributes.
However, existing identifiability and extrapolation guarantees are largely model-specific, with separate analyses in nonlinear ICA, causal representation learning, perturbation modeling, and related conditional latent variable models.
We introduce \emph{concept modulation models (CMMs)}, an attribute-indexed class of conditional generative models with structure \(A\to \modrv \to C\to X\), where attributes select modulators, modulators induce latent concept laws, and concepts generate observed features.
CMMs lift transition-based identifiability to conditional settings by showing that feature agreement on observed attributes induces a latent concept transition constrained by the CMM class.
We express these constraints through \emph{attribute potentials}, log-density ratios between attribute-conditioned concept laws, separating the generic lifting step from model-specific rigidity arguments.
The same potentials control extrapolation: agreement at unseen attributes holds exactly when the transported attribute-potential identities extend to those attributes.
This yields algebraic extrapolation criteria, identifies the common potential-based proof objects behind several existing identifiability and extrapolation results, and, when combined with the model-specific rigidity arguments in those works, recovers their stated conclusions.
\end{abstract}

\addtocontents{toc}{\protect\setcounter{tocdepth}{0}}

\section{Introduction}\label{sec:intro}

Identifiability in latent-variable representation learning asks when latent structure is not merely useful for prediction, but uniquely determined by the variation available in data, at least up to well-characterized ambiguities \citep{hyvarinen2019nonlinear,khemakhem2020variational,Scholkopf2021causal}.
This question is central to reliable generalization because data from observed conditions alone need not determine behavior under unseen ones: two models may agree on all observed conditions while disagreeing off-support \citep{damour2022underspecification}.
An identifiability guarantee offers a principled route around this obstacle: if the latent structure governing how attributes affect observations is identifiable, then recovering it can help justify generalization beyond the observed training conditions.

In this paper, we consider two problems central to this route in conditional generative models, where an observed feature \(X\) varies with an attribute \(A \in \cA\), such as an environment, intervention, perturbation, or conditioning input.
\emph{Identifiability} asks what aspects of the latent structure are determined, up to allowable ambiguities, by the conditional feature distributions \(p(x \mid a)\) observed at attributes \(a\in\cAo\), where $\cAo \subseteq \cA$ may be a very small subset of $\cA$.
\emph{Extrapolation} asks when agreement on the observed conditional distributions, together with the structural assumptions of the model class, forces agreement on \(p(x \mid a')\) at unseen attributes \(a'\in\cA\setminus\cAo\).

Existing answers to these questions have largely been developed on a per-model basis.
Nonlinear ICA and identifiable VAEs use auxiliary or conditioning variables to identify latent sources \citep{hyvarinen2019nonlinear,khemakhem2020variational}.
Causal representation learning uses environments or interventions to recover latent causal variables and causal structure \citep{ahuja2023interventional,squires2023linear,von2023nonparametric,varici2024general,varici2025score}.
Perturbation models study when learned latent responses can extrapolate to unseen perturbations \citep{von2025representation}.
Although these results share a common reliance on structured variation across attributes, their assumptions and conclusions are usually stated in model-specific terms.

Several recent works provide unifying perspectives, but their scope is different from ours.
For example, \citet{khemakhem2020variational} unify VAEs and nonlinear ICA through condition-dependent latent priors, while \citet{YaoRancatiCadeiFumeroLocatello2025unifying} unify causal representation learning through an invariance principle.
\citet{Reizinger2025identifiable} study identifiability through exchangeable mechanisms or related structural assumptions.
These frameworks clarify important families of identifiability results, but they primarily address when latent representations are identifiable within particular structural regimes.
Our goal is complementary: we address a wide variety of structural regimes at once by exposing the contrastive objects that underlie several identifiability results, and we use these objects to derive conditions under for extrapolation.
We provide a more detailed survey and comparison in \Cref{appendix:related-works}.

To give a common formulation of these different structural regimes, we introduce a shared structural form for attribute-conditioned generation: \(A\to\modrv\to C\to X\).
Here, the attribute \(A\) indexes a latent modulator \(\modrv\), the modulator specifies a latent concept distribution, and the concept \(C\) generates the observed feature \(X\).
We call such models \emph{concept modulation models} (CMMs).
The modulator separates attribute-specific indexing from the shared rule that maps modulators to concept distributions, thereby tying conditional concept laws together across attributes.

Our analysis builds on the transition-based identifiability perspective of \citet{squires2026unifying}, which reduces latent identifiability to characterizing concept-space transitions compatible with a model class.
CMMs are a conditional lift of this perspective: for each fixed attribute, a CMM induces a latent concept generative model, while the CMM class ties the family of attribute-conditioned concept laws through a shared modulation mechanism.
This attribute-indexed structure raises an extrapolation question absent from the unconditional setting: when does agreement of feature distributions on observed attributes force agreement at unseen attributes?

Our characterization of feature equivalence is expressed through the log-density ratio \(\log p(c \mid a)-\log p(c \mid a_0)\), which we call the \emph{attribute potential}.
Attribute potentials remove terms shared across attributes, isolating how the latent concept law changes with \(A\).
We show that feature-equivalent CMMs are related by a latent transition that preserves the observed attribute potentials, so identifiability reduces to determining which transitions remain compatible with the model class.

The same object also characterizes extrapolation.
Once agreement on observed attributes has been lifted to a latent transition, extrapolation is equivalent to the corresponding transported attribute-potential identities holding not only on \(\cAo\), but also at the unseen attributes of interest.
These conditions recover recent guarantees for causal representation learning and perturbation modeling, and also yield new interaction-based extrapolation guarantees for structured attribute spaces.

Our contributions are as follows:
\begin{itemize}
    \item \textbf{Concept modulation models:}
    We introduce CMMs, a conditional generative framework with graphical structure \(A\to\modrv\to C\to X\) (\Cref{def:cmm,def:cmm-class}). 
    This separates attribute-specific indexing from the shared modulation mechanism and captures several weakly supervised latent-variable settings.

    \item \textbf{Conditional transition-based identifiability:}
Building on the transition-based framework of \citet{squires2026unifying}, we prove a conditional lifting theorem for CMMs: feature equivalence on observed attributes induces a latent concept transition compatible with both components of the model class (\Cref{thm:intersection-individual}).
We then characterize the concept-side compatibility condition through preservation of attribute potentials (\Cref{thm:tq-characterization}).
    
    \item \textbf{Extrapolation via attribute potentials:}
    We prove that, once observed feature agreement has been lifted to a latent transition, agreement at unseen attributes is equivalent to extension of the transported attribute-potential identities (\Cref{thm:extrapolation-delta-iff}).
    This yields algebraic extrapolation criteria, recovers perturbation extrapolation guarantees as a special case, and gives interaction-based extrapolation results for structured attribute spaces.
    
\end{itemize}

\section{A unifying framework: Concept modulation models}
\label{sec:cmm}

\subsection{Preliminaries and notation}

Throughout, all spaces are standard Borel spaces.
For measures $\mu$, $\mu'$ defined on a space $\cU$, we write $\mu \ll \mu'$ if $\mu$ is absolutely continuous with respect to $\mu'$, i.e., $\mu'(B) = 0 \implies \mu(B) = 0$, and write $\mu \sim \mu'$ if $\mu \ll \mu'$ and $\mu' \ll \mu$.
We say that a condition holds $\mu$-a.e. for $u \in \cU$ if the set of points $u$ for which the condition fails has $\mu$-measure zero. 
We write $\mu \otimes \mu'$ for the product measure of $\mu$, $\mu'$.

For a space \(\cU\), let \(\cP(\cU)\) denote the set of probability measures on \(\cU\).
For spaces \(\cU,\cV\), let \(\cF(\cU\to\cV)\) denote the set of measurable maps
from \(\cU\) to \(\cV\), and let \(\MK(\cU\to\cV)\) denote the set of Markov kernels
from \(\cU\) to \(\cV\).
For measurable $\cU' \subseteq \cU$, we denote by $\bK_{\cU'}$ the restriction of $\bK$ to $\cU'$.
Given a measurable map \(f\in\cF(\cU\to\cV)\), we let \(\Push_f\colon \cP(\cU)\to\cP(\cV)\)
denote the associated pushforward operator, i.e., $(\Push_f \mu)(B) \defeq \mu(f^{-1}[B])$ for all $\mu\in\cP(\cU)$ and measurable $B\subseteq \cV$.
Given a Markov kernel \(\bK\in\MK(\cU\to\cV)\), we define its pushforward of measure $\mu \in \cP(\cU)$ by $(\bK\mu)(B) \defeq \int_{\cU} \bK(B\mid u)\,\mu(du)$.

As a running example throughout this paper, we consider attribute-conditioned distributions induced by the process
\begin{equation}\label{eq:running-example}
    p(c\mid a) \propto q(c)\exp\{\inner{f(a)}{c}\},
    \qquad
    X=g(C),
\end{equation}
where $f$ and $g$ are unknown.
Here, the attribute \(a\) is mapped to the modulator \(\modval = f(a)\), the modulator $\modval$ changes the latent concept law through an exponential-family tilt $q(c)\exp\{ \inner{\cdot}{c} \}$, and the concept is mapped to the observation by \(g\).

\subsection{Concept modulation models}

A concept modulation model (CMM) factors a conditional generative model through \(A \to \modrv \to C \to X\), where \(A \in \cA\) is an \defword{attribute}, \(\modrv \in \modclass\) is a \defword{modulator}, \(C \in \cC\) is a latent \defword{concept}, and \(X \in \cX\) is an observed \defword{feature}. 
Intuitively, the attribute selects a modulator, the modulator specifies a distribution over concepts, and the concept is mapped to the feature space.
Here, \(\cA\), \(\modclass\), \(\cC\), and \(\cX\) are standard Borel spaces, referred to as the attribute, modulator, concept, and feature spaces, respectively. 
In the running example \eqref{eq:running-example}, the $A \to \modrv$ stage corresponds to the map \(f\), the $\modrv \to C$ stage corresponds to the exponential tilting, and the $C \to X$ stage corresponds to the observation map \(g\).

We now formalize this generative structure at the level of Markov kernels.
\begin{definition}\label{def:cmm}
    A \defword{concept modulation model} \(\M\) on $(\cA,\modclass,\cC,\cX)$ is a tuple $\M=(\bQ,\bB,\bK)$, where \(\bQ \in \MK(\cA\to\modclass)\) is the \defword{indexing kernel}, \(\bB \in \MK(\modclass\to\cC)\) is the \defword{concept modulation kernel}, and \(\bK \in \MK(\cC\to\cX)\) is the \defword{mixing kernel}.
\end{definition}

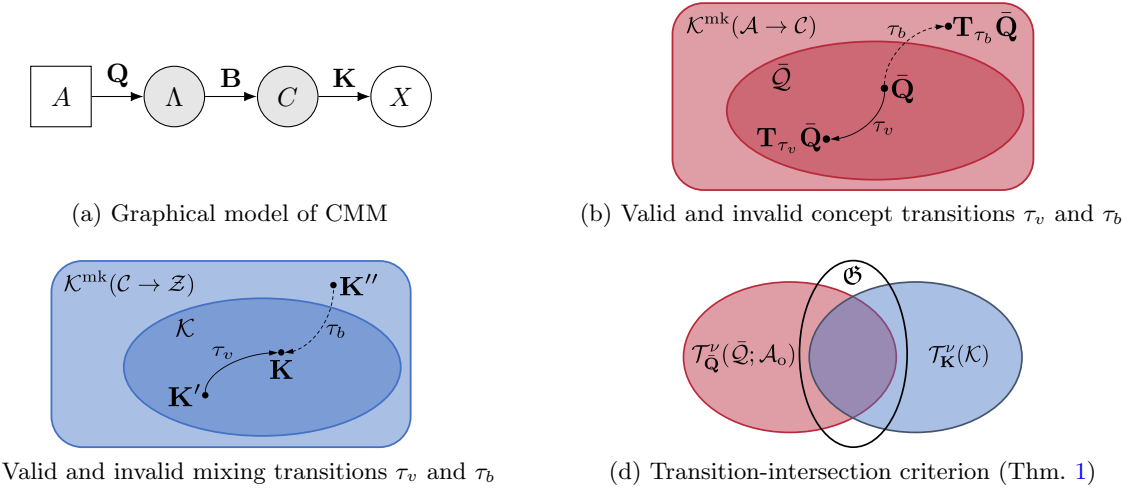
\begin{figure*}[t]
\centering

\def\panelW{5.0}
\def\panelH{2.5}

\newlength{\PanelCellW}
\setlength{\PanelCellW}{0.46\textwidth}

\newlength{\PanelContentHeight}
\setlength{\PanelContentHeight}{2.5cm}

\newlength{\PanelLabelHeight}
\setlength{\PanelLabelHeight}{0.75cm}

\newlength{\PanelColGap}
\setlength{\PanelColGap}{1.8em}

\newcommand{\PanelGraphic}[1]{%
  \includegraphics[
    width=0.95\linewidth,
    height=\PanelContentHeight,
    keepaspectratio
  ]{#1}%
}

\newcommand{\PanelCell}[2]{%
  \begin{minipage}[t]{\PanelCellW}
    \centering
    \begin{minipage}[c][\PanelContentHeight][c]{\linewidth}
      \centering
      #1
    \end{minipage}%
    \par\vspace{0.35em}%
    \begin{minipage}[t][\PanelLabelHeight][t]{\linewidth}
      \centering
      {\small #2}
    \end{minipage}%
  \end{minipage}%
}

\begin{tabular}{@{}c@{\hspace{\PanelColGap}}c@{}}

\PanelCell{%
\begin{tikzpicture}[
  >=Latex,
  every node/.style={font=\normalsize},
  var/.style={circle, draw, minimum size=8mm, inner sep=0pt},
  sqvar/.style={rectangle, draw, minimum size=8mm, inner sep=0pt},
  shaded/.style={fill=gray!20}
]
  \path[use as bounding box] (0,0) rectangle (\panelW,\panelH);

  \def\hgap{1.5}
  \pgfmathsetmacro{\gcenterX}{0.5*\panelW}
  \pgfmathsetmacro{\gcenterY}{0.5*\panelH}

  \node[sqvar] (a) at ({\gcenterX-1.5*\hgap},\gcenterY) {$A$};
  \node[var, shaded] (lam) at ({\gcenterX-0.5*\hgap},\gcenterY) {$\modrv$};
  \node[var, shaded] (c) at ({\gcenterX+0.5*\hgap},\gcenterY) {$C$};
  \node[var] (x) at ({\gcenterX+1.5*\hgap},\gcenterY) {$X$};

  \draw[->] (a) -- node[midway, above] {$\bQ$} (lam);
  \draw[->] (lam) -- node[midway, above] {$\bB$} (c);
  \draw[->] (c) -- node[midway, above] {$\bK$} (x);
\end{tikzpicture}%
}{(a) Graphical model of CMM}

&

\PanelCell{%
  \PanelGraphic{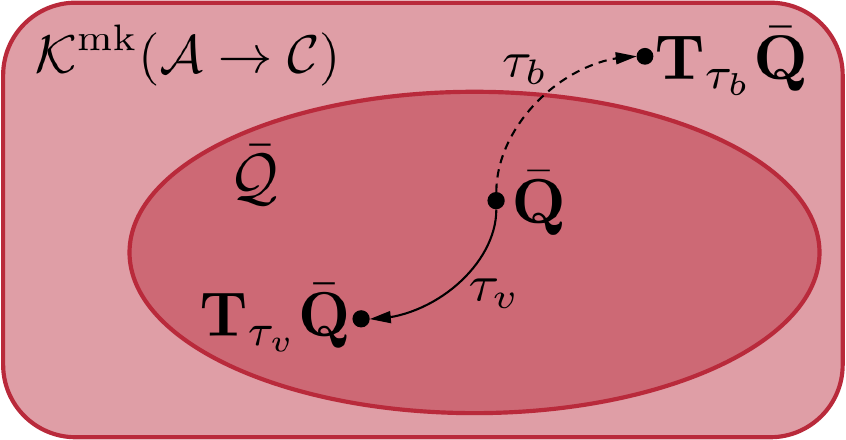}%
}{(b) Valid and invalid concept transitions $\tau_v$ and $\tau_b$}

\\

\PanelCell{%
  \PanelGraphic{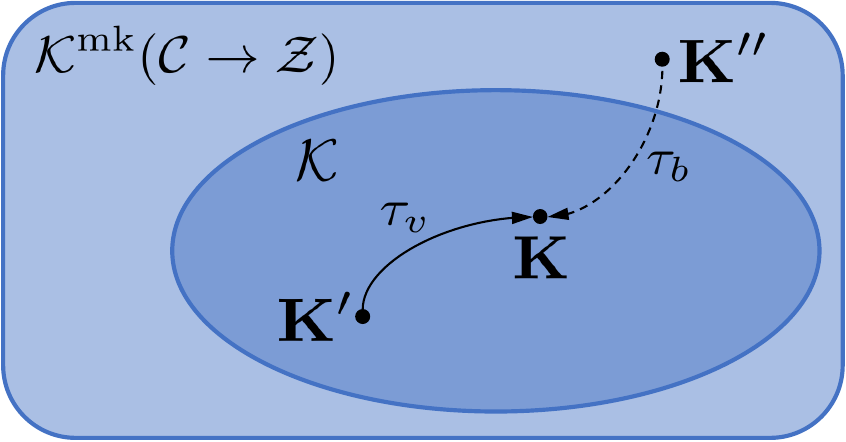}%
}{(c) Valid and invalid mixing transitions $\tau_v$ and $\tau_b$}

&

\PanelCell{%
  \PanelGraphic{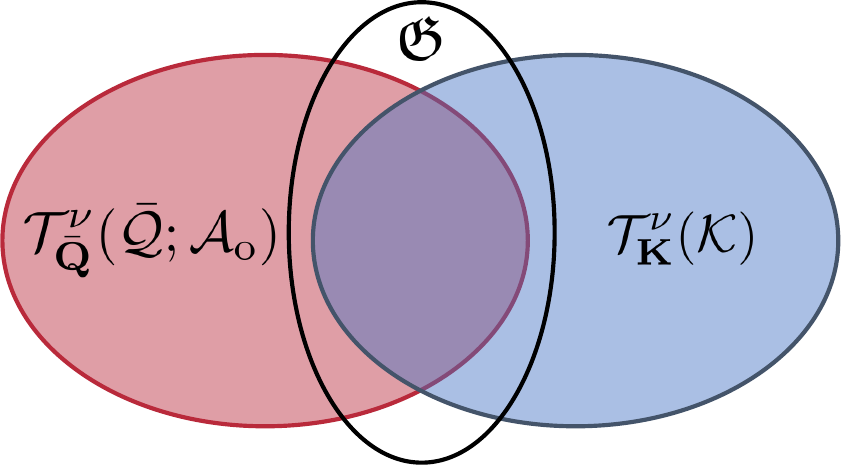}%
}{(d) Transition-intersection criterion (\Cref{thm:intersection-individual})}

\end{tabular}

\vspace{-1.0em}
\caption{
\textbf{Overview of CMMs and transition constraints.}
\textit{(a)} The observed attribute \(A\) influences a modulation variable \(\modrv\), which modulates the latent concept \(C\), which generates the observation \(X\).
\textit{(b,c)} Valid concept and mixing transitions are concept-space transformations that can be absorbed into the corresponding model class. 
Arrows in (b) denote equivalence after restriction to \(\cAo\), while arrows in (c) read \(\bK \refmeq \bK'\Push_{\tau_v} \refmeq \bK''\Push_{\tau_b}\).
\textit{(d)} The transition-intersection criterion selects transitions valid on both sides; if their intersection is contained in \(\frakG\), the CMM is identifiable up to \(\sim_\frakG\) (\Cref{thm:intersection-individual}).
}
\label{fig:cmm}
\end{figure*}
We denote by \(\bP^\M\defeq \bK\bB\bQ\in\MK(\cA\to\cX)\) the \defword{feature generation kernel} induced by $\M$.
A graphical representation of \Cref{def:cmm} is shown in \Cref{fig:cmm}: the indexing kernel \(\bQ\) specifies how attributes index modulators, while the fixed concept modulation kernel \(\bB\) specifies how each modulator induces a concept distribution.
A CMM \emph{class} specifies candidate models with a shared concept-modulation mechanism. 
We therefore fix the kernel \(\bB\), which maps modulators to concept distributions, while allowing the indexing kernel \(\bQ\) and the mixing kernel \(\bK\) to vary.
\begin{definition}
\label{def:cmm-class}
    A \defword{concept modulation model class (CMM class)} is a product class
    \[
        \cM = \cQ \times \{\bB\} \times \cK,
        \quad
        \text{where}
        \quad
        \cQ \subseteq \MK(\cA\to\modclass)
        \quad
        \text{and}\quad
        \cK \subseteq \MK(\cC\to\cX)
    \]
    so that all models in \(\cM\) share the same concept modulation kernel \(\bB\).
    Two models \(\M=(\bQ,\bB,\bK)\) and \(\M'=(\bQ',\bB,\bK')\) in the CMM class \(\cM\) are \defword{feature equivalent on \(\cAo \subseteq \cA \)} if $\bP^\M_{\cAo}=\bP^{\M'}_{\cAo}$, equivalently, if $\bK\bB\bQ(\cdot\mid a)=\bK'\bB\bQ'(\cdot\mid a)$ for all $a \in \cAo$.
\end{definition}

Given \(\cM=\cQ\times\{\bB\}\times\cK\), we denote the induced concept-kernel class by \(\bB\cQ \defeq\setbuild{\bB\bQ}{\bQ\in\cQ}\subseteq\MK(\cA\to\cC)\).
When \(\bB\) is clear from context, we write \(\bar\bQ\) and \(\bar{\cQ}\) for \(\bB\bQ\) and \(\bB\cQ\), respectively; throughout, any notation \(\bar{\bQ}\) or \(\bar{\cQ}\) is understood relative to this fixed \(\bB\).

As we will see, many existing identifiability results, including nonlinear ICA \citep{hyvarinen2019nonlinear,khemakhem2020variational,khemakhem2020ice}, causal representation learning \citep{squires2023linear,Buchholz2023learning,von2023nonparametric,varici2025score}, and other works \citep{von2025representation,SchmidtSchneiderBethge2025equivariance} can be expressed in this class-level framework. 
\Cref{tab:related_work_objects_full,tab:related_work_operators_full} in \Cref{appendix:tables} illustrate corresponding variable-level and operator-level dictionaries for selected results.

We close this section by showing how the running example fits into the CMM framework.
 \begin{example}\label{ex:running-example-setup}
    For the running example in \eqref{eq:running-example}, take
    \(\cA=\bbR^m\), \(\modclass=\bbR^k\), \(\cC=\bbR^k\), and \(\cX=\bbR^d\). Then
    \[
        \bQ=\Push_f,
        \qquad
        \bB(dc\mid\modval)
        =\frac{q(c) \exp(\inner{\modval}{c})}{Z(\modval)}\,dc,
        \qquad
        \bK=\Push_g,
    \]
    where $q(c) > 0$ for all $c \in \cC$ and \(Z(\modval)=\int q(c) \exp(\inner{\modval}{c})\,dc < \infty\). 
    A corresponding CMM class is
    \(\cM=\cQ\times\{\bB\}\times\cK\), with $\cQ=\setbuild{\Push_f}{f\in\cF}$ and $\cK=\setbuild{\Push_g}{g\in \cG}$, where \(\cF\subseteq\cF(\cA\to\modclass)\) and \(\cG \subseteq \cF(\cC \to \cX) \) are classes of functions, with \(\cG\) restricted to injective smooth maps.
\end{example}

\section{Conditional transition-based identifiability and attribute potentials}\label{sec:identifiability}

Given a CMM class, we now show how identifiability can be recovered through a common transition-intersection argument, generalizing \citet{squires2026unifying} to the conditional generative setting.
The key idea is to view feature equivalence as evidence for a latent transition \(\tau\) on concept space.
For CMMs, the transition must not take the model outside the model class: after adjusting for \(\tau\), the resulting mixing kernel and attribute-conditioned concept laws must still be realizable in the model class.
Identifiability up to a transition group \(\frakG\) then follows when every transition satisfying these compatibility requirements lies in \(\frakG\), as formalized in \Cref{thm:intersection-individual}.
The compatibility condition for the attribute-conditioned concept laws can then be characterized by log-density ratios, as in \Cref{thm:tq-characterization}.

\subsection{A transition-intersection criterion}

We first fix the reference-measure setting in which latent transitions are defined. 
Let \(\refm\) be a \(\sigma\)-finite reference measure on \(\cC\), and restrict attention to concept laws dominated by \(\refm\). 
For kernels \(\bK,\bK'\in\MK(\cC\to\cX)\), write \(\bK\refmeq\bK'\) if \(\bK\mu=\bK'\mu\) for all \(\mu\ll\refm\). 
A \defword{latent concept transition} is a measurable automorphism of concept space, defined modulo \(\refm\)-null sets:
\[
    \Aut_{\refm}(\cC)
    \defeq
    \setbuild{\tau \colon \cC \to \cC}{
    \tau \text{ is invertible up to } \refm\text{-null sets, }
    \Push_\tau \refm \sim \refm,\ \text{and } \Push_{\tau^{-1}} \refm \sim \refm } .
\]
Thus, elements of \(\Aut_{\refm}(\cC)\) are precisely the concept-space transformations that preserve the \(\refm\)-null sets in both directions.
For example, when \(\cC=\bbR^d\) and \(\refm\) is Lebesgue measure, this class contains the usual diffeomorphisms.

We next impose the structural condition that allows feature equivalence to be lifted to such a transition. Following the transition-based perspective of \citet{squires2026unifying}, we require the mixing class to factor through an injective embedding of concept distributions.

\begin{definition}[\(\refm\)-Blackwell reducible mixing class]
\label{def:mu-blackwell-reducible}
A mixing class \(\cK\subseteq\MK(\cC\to\cX)\) is \defword{\(\refm\)-Blackwell reducible} if there exist a standard Borel space \(\tcX\), a class \(\cG\) of bimeasurable embeddings \(g\colon\cC\to\tcX\), and a shared kernel \(\widetilde{\bK}\in\MK(\tcX\to\cX)\) such that every \(\bK\in\cK\) admits some \(g\in\cG\) with \(\bK\refmeq\widetilde{\bK}\Push_g\), and the map \(\xi\mapsto \widetilde{\bK}\xi\) is injective on \(\setbuild{\Push_g\mu}{g\in\cG,\ \mu \ll \refm}\).
\end{definition}

The injectivity of the shared kernel \(\widetilde{\bK}\) means that if two distributions on the intermediate space \(\tcX\) give the same observed feature law after applying \(\widetilde{\bK}\), then those distributions on \(\tcX\) must already be equal.
Thus, after writing \(\bK\refmeq\widetilde{\bK}\Push_g\), the relevant concept distribution is \(\Push_g\mu\) with the concept law \(\mu \ll \refm\).
This condition is what lets equality of feature laws lift to equality of $g$-pushed concept laws, and ultimately produces a transition \(\tau\in\Aut_{\refm}(\cC)\) in \Cref{thm:intersection-individual}.

Intuitively, a transition is valid if its effect can be absorbed without leaving the model class.
On the mixing side, this means changing only the mixing kernel, so that \(\bK'\refmeq\bK\Push_\tau\) for some \(\bK'\in\cK\).
On the concept side, this means changing only the induced concept kernel, so that \(\bar\bQ(\cdot\mid a)=\Push_\tau\bar\bQ'(\cdot\mid a)\) for all \(a\in\cAo\) and some \(\bar\bQ'\in\bar\cQ\).

\begin{definition}[Valid transition sets]\label{def:transition-sets}
    For \(\bar\bQ \in \bar\cQ\), define the \defword{valid concept transitions} by
    \[
        \cT_{\bar\bQ}^{\refm}(\bar\cQ;\cAo)
        \defeq
        \setbuild{
        \tau\in\Aut_{\refm}(\cC)}{
        \exists \bar\bQ' \in \bar\cQ\text{ such that }
        \bar\bQ(\cdot\mid a) = \Push_\tau\bar\bQ'(\cdot\mid a)
        \ \text{ for all } a\in\cAo
        }.
    \]
    For \(\bK\in\cK\), define the \defword{valid mixing transitions} by
    \[
        \cT_{\bK}^{\refm}(\cK)
        \defeq
        \setbuild{
        \tau\in\Aut_{\refm}(\cC)
        }{
        \exists \bK'\in\cK\text{ such that }\bK'\refmeq\bK\Push_\tau
        }.
    \]
\end{definition}
Panels (b) and (c) of \Cref{fig:cmm} illustrate the valid transition sets.
We adopt the notion of Blackwell equivalence proposed by \citet{squires2026unifying}, extending the notion defined in the unconditional setup to the attribute-conditional setup.

\begin{definition}[Blackwell equivalence, attribute-conditional version]
\label{def:blackwell-equivalence}
    When two models \(\M=(\bQ,\bB,\bK)\) and \(\M'=(\bQ',\bB,\bK')\) satisfy $\bK'\refmeq\bK\Push_\tau$ and $\bar\bQ(\cdot\mid a)=\Push_\tau\bar\bQ'(\cdot\mid a)$ for all $a\in\cAo$ with $\bar\bQ = \bB\bQ$ and $\bar\bQ' = \bB\bQ'$, we say $\M$ and $\M'$ are \defword{Blackwell equivalent via $\tau$}.
    When $\tau \in \frakG$ for a transition group $\frakG$, we write \(\M\sim_{\frakG}\M'\) on \(\cAo\).
\end{definition}
Given that \(\frakG\) is a group, \(\sim_{\frakG}\) is an equivalence relation; see \Cref{appendix:proof-equiv}.
We next impose a common-support condition to keep concept transitions and density ratios well defined.
An analogous condition is assumed by \citet{squires2026unifying}, and it holds for the main continuous examples below.

\begin{definition}[Common \(\refm\)-support]\label{def:ref-support}
We say that the induced concept-kernel class \(\bar\cQ = \bB\cQ\) has \defword{common \(\refm\)-support} if, for every \(\bar\bQ\in\bar\cQ\) and every \(a\in\cA\), it holds that \(\bar\bQ(\cdot\mid a)\ll\refm\) with density $p_{\bar\bQ} \defeq d\bar\bQ(\cdot\mid a) / d\refm$ satisfying \(0<p_{\bar\bQ}(c\mid a)<\infty\) for \(\refm\)-almost every \(c\).
\end{definition}

The following theorem is a CMM-level analogue of the transition-intersection theorem of \citet{squires2026unifying}, visualized in panel (d) of \Cref{fig:cmm}.
Its CMM-specific content is that feature equivalence induces a transition that simultaneously transports the \emph{entire} observed set of attribute-conditioned concept laws, rather than a single distribution.
\Cref{thm:tq-characterization} later gives the density-level form of this concept-side compatibility.

\begin{restatable}[CMM-lift of the intersection theorem]{theorem}{intersection}
\label{thm:intersection-individual}
Let \(\cM=\cQ\times\{\bB\}\times\cK\) be a CMM class, and let \(\cAo\neq\varnothing\).
Assume that \(\cK\) is \(\refm\)-Blackwell reducible and that the induced concept-kernel class \(\bar\cQ = \bB\cQ\) has common \(\refm\)-support.
Then, for any \(\M,\M'\in\cM\) feature-equivalent on \(\cAo\), there exists 
\[
    \tau\in\cT_{\bK}^\refm(\cK)\cap\cT_{\bar\bQ}^\refm(\bar\cQ;\cAo)
\]
such that \(\M\) and \(\M'\) are Blackwell-equivalent through \(\tau\) on \(\cAo\).
If, moreover, \(\cT_{\bK}^\refm(\cK)\cap\cT_{\bar\bQ}^\refm(\bar\cQ;\cAo)\subseteq\frakG\) for a transition group \(\frakG\subseteq\Aut_{\refm}(\cC)\), then \(\M\) is identifiable up to \(\sim_{\frakG}\) in \(\cM\) on \(\cAo\), i.e., every \(\M'\in\cM\) feature-equivalent to \(\M\) on \(\cAo\) satisfies \(\M\sim_{\frakG}\M'\).
\end{restatable}

\begin{proof}[Proof sketch]
We provide a proof sketch; \Cref{appendix:proof-intersection} gives the full proof under a condition weaker than the one introduced in \Cref{def:ref-support}.
Let \(\M=(\bQ,\bB,\bK)\) and \(\M'=(\bQ',\bB,\bK')\) be feature-equivalent on \(\cAo\). By \(\refm\)-Blackwell reducibility, write \(\bK\refmeq \widetilde{\bK}\Push_g\) and \(\bK'\refmeq \widetilde{\bK}\Push_{g'}\), where \(g\) and \(g'\) are bimeasurable injective maps and \(\widetilde{\bK}\) is injective on the relevant pushed-forward laws.
The common-support condition then makes \(\tau=g^{-1}\circ g' \in \Aut_\refm(\cC)\) a well-defined concept transition modulo \(\refm\)-null sets.
We therefore obtain \(\bK'\refmeq\bK\Push_\tau\) and \(\bB\bQ(\cdot\mid a)=\Push_\tau\bB\bQ'(\cdot\mid a)\) for all \(a\in\cAo\).
Hence \(\tau\in\cT_{\bK}^{\refm}(\cK)\cap\cT_{\bar\bQ}^{\refm}(\bar\cQ;\cAo)\), and if this intersection lies in \(\frakG\), the ambiguity is up to \(\sim_{\frakG}\).
\end{proof}

\subsection{Characterizing concept-side transitions}

We now characterize the concept-side valid transitions at the density level.
This refines the transition-intersection criterion by isolating the common proof object behind model-specific identifiability arguments.
Instead of working directly with equality of pushed concept laws, we use density ratios between attribute-conditioned concept distributions; their logarithms are the \emph{attribute potentials}.

\begin{definition}
\label{def:attribute-potential}
Suppose that \(\bB\) and \(\bQ\) together satisfy \Cref{def:ref-support}.
For \(a,a'\in\cA\), define the \defword{attribute potential} of $\bar\bQ = \bB\bQ$ by \(\Delta_{a,a'}^{\bar\bQ}(c)\defeq \log p_{\bar\bQ}(c\mid a)-\log p_{\bar\bQ}(c\mid a')\), which is well-defined \(\refm\)-a.e.
\end{definition}

The following theorem gives an equivalent characterization of the concept-side valid transition set \(\cT_{\bar\bQ}^{\refm}(\bar\cQ;\cAo)\).
This characterization will also be the main tool for extrapolation in \Cref{sec:extrapolation}, where we ask when preservation of observed attribute potentials forces preservation at unseen attributes.

\begin{restatable}[Density characterization of concept transitions]{theorem}{tqcharacterization}
\label{thm:tq-characterization}
Consider a fixed concept modulation kernel $\bB$, and assume that $\bar\cQ' = \bB\cQ$ has common $\refm$-support as defined in \Cref{def:ref-support}.
Let \(\bQ,\bQ' \in \cQ\) be two indexing kernels.
Let \(\cA'\subseteq\cA\), \(\tau\in\Aut_{\refm}(\cC)\), and define $r_{\tau^{-1}} \defeq  d(\Push_{\tau^{-1}}\refm) / d\refm$.
For any fixed anchor \(a_0\in\cA'\) and $\bar\bQ = \bB\bQ$, $\bar\bQ' = \bB\bQ'$, the following are equivalent:
\begin{enumerate}[label=(\roman*),leftmargin=*,itemsep=0pt,topsep=0pt]
    \item \(\bar\bQ_{\cA'}=\Push_\tau\bar\bQ'_{\cA'}\).
    \item For every \(a\in\cA'\), $p_{\bar\bQ'}(c\mid a) = p_{\bar\bQ}(\tau(c)\mid a)r_{\tau^{-1}}(c)$ holds for $\refm$-a.e. $c$.
    \item The anchored density identity $p_{\bar\bQ'}(c\mid a_0)
        =
        p_{\bar\bQ}(\tau(c)\mid a_0)r_{\tau^{-1}}(c)$ holds for $\refm$-a.e. $c$, and for every \(a\in\cA'\), $\Delta_{a,a_0}^{\bar\bQ'}(c) = \Delta_{a,a_0}^{\bar\bQ}(\tau(c))$ holds $\refm$-a.e. $c$.
\end{enumerate}
Consequently, \(\tau\in\cT_{\bar\bQ}^\refm(\bar\cQ;\cAo)\) if and only if there exists \(\bQ'\in\cQ\) such that these equivalent conditions hold with \(\cA'=\cAo\).
\end{restatable}

\begin{proof}[Proof sketch]
The equivalence between (i) and (ii) is the Radon--Nikodym chain rule applied to \(\Push_{\tau^{-1}}\).
Under the positivity assumption, (ii) implies (iii) by taking \(a=a_0\), then taking logs of the identities in (ii) for \(a\), \(a_0\) and subtracting to cancel \(r_{\tau^{-1}}(c)\).
Conversely, if (iii) holds, multiplying the anchored density identity by the exponential of the attribute potential gives (ii).
The final statement follows from the definition of \(\cT_{\bar\bQ}^\refm(\bar\cQ;\cAo)\).
\end{proof}

Notably, statement \textit{(iii)} of \Cref{thm:tq-characterization} characterizes the concept-side valid transition set \(\cT_{\bar\bQ}^{\refm}(\bar\cQ;\cAo)\) in terms of anchored density identities and attribute-potential preservation, avoiding the intractable $r_{\tau^{-1}}$ term.
We now illustrate how \Cref{thm:intersection-individual,thm:tq-characterization} recover an identifiability result for the running example.

\begin{example}[Running example, affine recovery]
\label{ex:running-example-theorems}
Consider the running example (\Cref{ex:running-example-setup}). 
Taking \(\widetilde\cX=\cX\) and \(\widetilde{\bK}=\Push_{\mathrm{id}_{\cX}}\), the class \(\cK\) is \(\refm\)-Blackwell reducible with \(\refm\) taken to be Lebesgue measure on \(\bbR^k\).

Fix \(\M=(\bQ,\bB,\bK)\in\cM\), with \(\bQ=\Push_f\) and
\(\bK=\Push_g\).
Fixing \(a_0\in\cAo\), the attribute potential is
\begin{equation}\label{eq:running-example-attribute-potential}
    \Delta^{\bar\bQ}_{a,a_0}(c)
    =
    \inner{f(a)-f(a_0)}{c}
    -
    \{\log Z(f(a))-\log Z(f(a_0))\},
\end{equation}
and \Cref{thm:tq-characterization} implies that for any \(\tau\in\cT_{\bar\bQ}^{\refm}(\bar\cQ;\cAo)\) there must exist \(\bQ'\in\cQ\) such that $\Delta_{a, a_0}^{\bar\bQ'}(c) = \Delta_{a, a_0}^{\bar\bQ}(\tau(c))$.
In particular, if $\Span\setbuild{f(a)-f(a_0)}{a\in\cAo}=\bbR^k$, then preservation forces \(\tau\) to be affine, as we show in \Cref{app:running-example-details}. 
Hence $\cT_{\bar\bQ}^{\refm}(\bar\cQ;\cAo) \subseteq \mathrm{Affine}(k)\cap\Aut_{\refm}(\bbR^k)$, where $\mathrm{Affine}(k)$ is the set of affine transformations on $\bbR^k$.
Therefore,
\[
    \cT_{\bK}^{\refm}(\cK)\cap\cT_{\bar\bQ}^{\refm}(\bar\cQ;\cAo)
    \subseteq
    \cT_{\bar\bQ}^{\refm}(\bar\cQ;\cAo)
    \subseteq
    \frakG \coloneqq \mathrm{Affine}(k)\cap\Aut_{\refm}(\bbR^k).
\]
Since the exponential-family concept laws have positive densities, the support
condition in \Cref{thm:intersection-individual} holds. 
Therefore, \Cref{thm:intersection-individual} implies that the running example is
identifiable up to $\sim_{\frakG}$, i.e., invertible affine transformations of the concept space.
\end{example}

Many existing identifiability results impose richness assumptions under which the observed contrasts have enough rank to identify the entire latent representation up to the target ambiguity class.
The attribute-potential characterization also describes what remains identifiable when this rank condition fails: the observed contrast equations constrain only the component of the latent transition visible through the observed contrast span.

\begin{remark}[Partial identifiability under incomplete contrast span]
Consider the CMM class from \Cref{ex:running-example-setup,ex:running-example-theorems}, but let $V=\Span\setbuild{f(a)-f(a_0)}{a\in\cAo}\subsetneq\bbR^k$.
Then feature equivalence need not force the latent transition \(\tau\) to be
affine on all of \(\bbR^k\). 
Nevertheless, the projection identity derived in \Cref{app:running-example-details}, $\Pi_V\tau(c)=M\Pi_{V'}c+b$, where $\Pi_V$ and $\Pi_{V'}$ are projection matrices onto $V$ and $V'$, shows that the component of the projected concept in the observed contrast subspace is still identified up to an affine map. 
If the observed contrast subspace is fixed across the class, i.e., \(\Span\setbuild{f(a)-f(a_0)}{a\in\cAo}=V_\cF\) for every \(f\in\cF\), then the latent transition \(\tau\) belongs to a group whose elements may be arbitrary off \(V_\cF\) but must act affinely after projection onto \(V_\cF\).
\end{remark}


\section{Extrapolation via attribute potentials}\label{sec:extrapolation}

Extrapolation asks when predictive agreement on observed attributes extends beyond them: if two CMMs agree on conditional feature laws over \(\cAo\), when must they also agree on \(\cAex\supseteq\cAo\)?
Unlike identifiability, which characterizes latent transitions compatible with feature equivalence on \(\cAo\), extrapolation asks whether the induced transition also enforces agreement outside \(\cAo\).
For a transition \(\tau\) from \Cref{thm:intersection-individual}, with \(\bK'\refmeq\bK\Push_\tau\) and \(\bar\bQ_{\cAo}=\Push_\tau\bar\bQ'_{\cAo}\), the mixing side remains fixed as the attribute set expands. 
Therefore, extrapolation reduces to whether this concept-side relation extends from \(\cAo\) to \(\cAex\).
By \Cref{thm:tq-characterization}, this is equivalent to preserving attribute potentials, as formalized next.

\begin{restatable}[Attribute-potential characterization of extrapolation]{theorem}{extrapolationdeltaiff}
\label{thm:extrapolation-delta-iff}
Let \(\cM=\cQ\times\{\bB\}\times\cK\) be a CMM class, and let \(\cAo\subseteq\cAex\subseteq\cA\).
Assume that \(\cK\) is \(\refm\)-Blackwell reducible and that \(\bar\cQ = \bB\cQ\) has common \(\refm\)-support.
Fix \(a_0\in\cAo\), and let \(\M=(\bQ,\bB,\bK)\) and \(\M'=(\bQ',\bB,\bK')\) be feature-equivalent on \(\cAo\), and let \(\tau\) be a transition induced by \Cref{thm:intersection-individual}, so that \(\bK'\refmeq\bK\Push_\tau\) and \(\bar\bQ_{\cAo}=\Push_\tau\bar\bQ'_{\cAo}\) where $\bar\bQ = \bB\bQ$, $\bar\bQ' = \bB\bQ'$.
Then, for this transition \(\tau\),
\[
    \bP^\M_{\cAex}=\bP^{\M'}_{\cAex}
    \quad\Longleftrightarrow\quad
    \Delta_{a,a_0}^{\bar\bQ'}(c)=\Delta_{a,a_0}^{\bar\bQ}(\tau(c))
    \text{ for every }a\in\cAex,\ \refm\text{-a.e. }c.
\]
Thus, if the identities $\Delta_{a,a_0}^{\bar\bQ'}(c)=\Delta_{a,a_0}^{\bar\bQ}(\tau(c))$ that hold on \(\cAo\) extend over \(\cAex\), then \(\bP^\M_{\cAex}=\bP^{\M'}_{\cAex}\).
\end{restatable}

We defer the proof of \Cref{thm:extrapolation-delta-iff} to \Cref{appendix:proof-extrapolation-delta-iff}.
The remaining work is to provide class-specific arguments showing that the transported attribute-potential identities on \(\cAo\) force the corresponding identities on \(\cAex\).
This implication is not automatic: in the running exponential-family CMM with unrestricted indexing map \(f\colon\cA\to\bbR^k\), one may keep \(f(a)\) unchanged for all \(a\in\cAo\) while altering \(f(\aex)\) arbitrarily at an unseen attribute \(\aex\notin\cAo\).
The observed conditional feature laws are unchanged, but \(p(X\mid A=\aex)\) can change.
Thus extrapolation requires structural assumptions tying unseen attribute potentials to observed ones.
The following theorem gives one such condition, requiring affine dependence through a fixed attribute representation \(\varphi\).

\begin{restatable}[Attribute-potential preservation under affine attribute-potential differences]{theorem}{affineindexingextrapolation}
\label{thm:extrapolation-affine-indexing}
Let \(\cM=\cQ\times\{\bB\}\times\cK\) be a CMM class satisfying the assumptions of \Cref{thm:intersection-individual,thm:tq-characterization}.
Let \(\cAo\subseteq\cAex\subseteq\cA\), and fix \(a_0\in\cAo\).
Suppose there exists a fixed attribute representation \(\varphi\colon\cA\to\bbR^m\) such that, for every $\bar\bQ \in \bar\cQ = \bB\cQ$, there exists a measurable map \(D^{\bar\bQ}\colon\cC\times\cC\to\bbR^m\) satisfying
\[
    \Delta_{a,a_0}^{\bar{\bQ}}(c)-\Delta_{a,a_0}^{\bar{\bQ}}(c_0)
    =
    \inner{\varphi(a)-\varphi(a_0)}{D^{\bar\bQ}(c,c_0)}
\]
for every \(a\in\cAex\) and \(\refm \otimes \refm\)-a.e. \((c,c_0)\).
Let \(\M = (\bQ, \bB, \bK)\) and \(\M' = (\bQ', \bB, \bK')\) be feature-equivalent on \(\cAo\), and let \(\tau\) be the transition induced by \Cref{thm:intersection-individual}.
If \(\varphi(\cAex)\) is contained in the affine hull of \(\varphi(\cAo)\), then the transported attribute-potential identities extend from \(\cAo\) to \(\cAex\): for every \(a\in\cAex\), $\Delta_{a,a_0}^{\bar{\bQ}'}(c) = \Delta_{a,a_0}^{\bar{\bQ}}(\tau(c))$ for \(\refm\)-a.e. \(c\).
\end{restatable}

\begin{proof}[Proof sketch]
We defer the full proof to \Cref{appendix:proof-extrapolation-affine-indexing}.
For any \(\aex\in\cAex\), we can write \(\varphi(\aex)=\sum_i\alpha_i\varphi(a_i)\), where \(a_i\in\cAo\) and \(\sum_i\alpha_i=1\).
For any $\M \in \cM$, we have the centered potential \(\Delta_{a,a_0}^{\bar{\bQ}}(c)-\Delta_{a,a_0}^{\bar{\bQ}}(c_0)=\inner{\varphi(a)-\varphi(a_0)}{D^{\bar\bQ}(c, c_0)}\), so the same quantity at \(\aex\) is an affine combination of the centered potentials at the \(a_i\)'s, and the same relation holds for \(\bar{\bQ}'\).
Statement \textit{(iii)} of \Cref{thm:tq-characterization} transports the observed identities, hence the affine expansion gives \(\Delta_{\aex,a_0}^{\bar{\bQ}'}(c)-\Delta_{\aex,a_0}^{\bar{\bQ}'}(c_0)=\Delta_{\aex,a_0}^{\bar{\bQ}}(\tau(c))-\Delta_{\aex,a_0}^{\bar{\bQ}}(\tau(c_0))\).
The density identity at \(a_0\) and normalization at \(\aex\) together force \(\Delta_{\aex,a_0}^{\bar{\bQ}'}(c)=\Delta_{\aex,a_0}^{\bar{\bQ}}(\tau(c))\).
\end{proof}

The point of \Cref{thm:extrapolation-affine-indexing} is that extrapolation is driven by the linear structure of attribute potentials, not by linearity of the raw attribute space.
The representation \(\varphi\) may encode nonlinear or combinatorial structure, such as interaction terms among atomic attributes, as in factorial designs \citep{Fisher1935design} and combinatorial intervention models \citep{odonnell2008some,agarwal2023synthetic}.
Thus, the theorem lifts feature-equivalence on finitely many interactions to extrapolation guarantees over a much larger combinatorial attribute space, as stated in the following corollary.

\begin{restatable}[Interaction-basis extrapolation]{corollary}{interactionextrapolation}
\label{cor:interaction-extrapolation}
Let $\cM = \cQ \times \{\bB\} \times \cK$ satisfy the assumptions of \Cref{thm:extrapolation-delta-iff}.
Let \(\cA=2^{[m]}\), let \(\cH\subseteq 2^{[m]}\) be a finite downward-closed family containing \(\varnothing\), i.e., \(H\in\cH\) and \(U\subseteq H\) imply \(U\in\cH\), and set \(\cAo=\cH\).
Assume every \(\M\in\cM\) induces concept densities of the form $p^\M(c\mid S) =  \exp\left(\sum_{T\in\cH,\,T\subseteq S} h_T^\M(c)\right) / Z^\M(S)$.
Then feature equivalence on \(\cAo\) implies feature equivalence on all \(S\subseteq[m]\).
\end{restatable}

\begin{proof}[Proof sketch]    
We defer a full proof to \Cref{appendix:proof-interaction-extrapolation}.
Denoting \(\mathbf{1}\) by the indicator function, we have \(\Delta_{S,\varnothing}^\M(c)-\Delta_{S,\varnothing}^\M(c_0)=\inner{\varphi_\cH(S)-\varphi_\cH(\varnothing)}{D^\M(c,c_0)}\), where \(\varphi_\cH(S)=(\mathbf{1}\{T\subseteq S\})_{T\in\cH}\) and \(D_T^\M(c,c_0)=h_T^\M(c)-h_T^\M(c_0)\) for \(T\neq\varnothing\) and \(D_\varnothing^\M(c,c_0)=0\).
One can show that \(\varphi_\cH(S)\) is an affine combination of \(\setbuild{\varphi_\cH(U)}{U\in\cH}\) for any $S \subseteq [m]$.
\Cref{thm:extrapolation-affine-indexing,thm:extrapolation-delta-iff} yield desired result.
\end{proof}

We close the section by recording an implication of \Cref{thm:extrapolation-affine-indexing} for the running example.

\begin{example}[Running example: extrapolation]\label{ex:running-example-extrapolation}
For the running exponential-family CMM, assume that $f(a) = W_f \varphi(a)$ for a matrix \(W_f\) depending on \(f\) and a fixed attribute representation $\varphi$.
This satisfies the attribute-potential assumption of \Cref{thm:extrapolation-affine-indexing} with $D^{\bar\bQ}(c, c_0) = W_f^\top (c - c_0)$.
Consequently, by \Cref{thm:extrapolation-delta-iff}, feature equivalence on \(\cAo\) implies feature equivalence at every target \(\aex\in\cAex\) whose \(\varphi(\aex)\) is an affine combination of \(\setbuild{\varphi(a)}{a\in\cAo}\).
\end{example}

\section{Applications: Recoveries and examples}\label{sec:applications}

In this section, we apply the transition-and-potential characterizations from \Cref{thm:intersection-individual,thm:tq-characterization} and the extrapolation criteria from \Cref{thm:extrapolation-delta-iff,thm:extrapolation-affine-indexing} to representative model classes.
Identifiability follows by combining the induced attribute potentials with class-specific rigidity arguments.
Extrapolation follows by checking whether the transported potential identities observed on \(\cAo\) force the corresponding identities on target attributes in \(\cAex\).

\paragraph{Nonlinear ICA and iVAE.}
Nonlinear ICA with auxiliary variables \citep{hyvarinen2019nonlinear} and iVAE \citep{khemakhem2020variational} fit the CMM template by taking the attribute \(a\in\cA\) to be the auxiliary or conditioning variable, the concept \(c=(c_1,\ldots,c_k)\in\cC\) to be the latent source vector, and the feature \(x\in\cX\) to be the nonlinear observation.
In the conditionally independent source setting, the attribute potential decomposes as \(\Delta^{\bar\bQ}_{a,a_0}(c)=\sum_{i=1}^k[\log q_i(c_i\mid a)-\log q_i(c_i\mid a_0)]\), so the CMM contrast is exactly the coordinate-wise auxiliary-variable contrast used in nonlinear ICA and iVAE.
Under the variability assumptions of \citet{hyvarinen2019nonlinear}, preservation of these potentials rules out mixing across source coordinates, forcing \(\tau(c)=(\phi_1(c_{\pi(1)}),\ldots,\phi_k(c_{\pi(k)}))\) up to a permutation \(\pi\) and invertible scalar maps \(\phi_i\).
More details on recoveries are given in \Cref{appendix:proof-nonlinear-ica}.

\paragraph{Causal representation learning.}
Causal representation learning fits the CMM template by taking \(a\in\cA\) to be an environment or intervention label, \(c=(c_1,\ldots,c_k)\in\prod_{i=1}^k\cC_i\) to be the latent causal representation, and \(x\in\cX\) to be the observed representation.
We work in the common-support regime, so the induced joint concept laws have positive densities with respect to a common reference measure and the attribute potentials in \Cref{thm:tq-characterization} are well defined.
The modulator is the tuple of local mechanisms \(\modval=(p_i)_{i=1}^k\), indexed by a directed acyclic graph (DAG) \(G_{\modval}\), with \(\modclass=\bigsqcup_{G\in\mathrm{DAG}([k])}\prod_{i=1}^k\MK(\cC_{\pa_G(i)}\to\cC_i)\) and \(\bB(dc\mid\modval)=\prod_{i=1}^k p_i(c_i\mid c_{\pa_{G_{\modval}}(i)})\,dc\).
Here, $\mathrm{DAG}([k])$ is the set of DAGs on $[k]$, \(\pa_G(i)\) is the parent set of node \(i\) in $G$, \(\cC_S\coloneqq\prod_{i\in S}\cC_i\), and the indexing kernel \(\bQ=\Push_f\) maps each environment label to its active tuple of local mechanisms.

For an intervention label \(a\), let \(\operatorname{tar}(a)\subseteq[k]\) denote its target nodes, and write \(p_i^a\) for the local mechanism active at node \(i\) under \(a\): it satisfies \(p_i^a=p_i^0\) for \(i\notin\operatorname{tar}(a)\), while \(p_i^a\) may differ from \(p_i^0\) for \(i\in\operatorname{tar}(a)\).
Fixing an observational label \(a_0\in\cAo\) with mechanisms \((p_i^0)_{i=1}^k\), the attribute potential becomes
$\Delta_{a,a_0}^{\bar{\bQ}}(c) = \sum_{i\in\operatorname{tar}(a)}[\log p_i^a(c_i\mid c_{\pa_G(i)}) - \log p_i^0(c_i\mid c_{\pa_G(i)})]$.
Thus, the attribute potential records exactly the local mechanism contrasts induced by the intervention.
Across CRL variants, identifiability follows from showing that transitions preserving these contrasts must belong to the target ambiguity class: Gaussian CRL \citep{Buchholz2023learning}, interventional-to-observational density ratios in nonparametric CRL \citep{von2023nonparametric}, and gradients of environment ratios in score-based CRL \citep{varici2024linear,varici2025score}.

For extrapolation, the additional structure is supplied by the independent causal mechanisms (ICM) principle \citep{peters2017elements,Scholkopf2021causal}: interventions may replace selected local mechanisms while leaving the remaining mechanisms unchanged.
Let $J \subseteq [k]$ be a set of intervention labels such that \(\setbuild{a_j}{j \in J}\subseteq\cAo\) have pairwise disjoint targets, i.e., \(\operatorname{tar}(a_j)\cap\operatorname{tar}(a_{j'})=\varnothing\) for all \(j\neq j'\).
The associated composite intervention \(a_J\) is defined by \(p_i^{a_J}=p_i^{a_j}\) if \(i\in\operatorname{tar}(a_j)\) for the unique \(j \in J\), and \(p_i^{a_J}=p_i^0\) if \(i\notin\bigcup_{j \in J}\operatorname{tar}(a_j)\).
Under the ICM principle, non-target mechanisms cancel against the anchor and the target sets do not overlap, so the composite potential is additive: \(\Delta_{a_J,a_0}^{\bar{\bQ}}(c)=\sum_{j \in J}\Delta_{a_j,a_0}^{\bar{\bQ}}(c)\).
Therefore, if the transported identities \(\Delta_{a_j,a_0}^{\bar{\bQ}'}(c)=\Delta_{a_j,a_0}^{\bar{\bQ}}(\tau(c))\) hold for all \(j \in J\), additivity gives the same identity for \(a_J\).
By \Cref{thm:extrapolation-delta-iff}, feature equivalence on the observed environments then extrapolates to the composite intervention \(a_J\). 
We note that this is analogous to the intervention-generalization strategy of \citet{Bravo2023intervention}, which studies extrapolation under factorization assumptions on the interventional factor models.

\paragraph{Perturbation modeling.}
Perturbation modeling in \citet{von2025representation} fits the CMM template by taking the attribute \(a\in\cA = \bbR^K\) to be a perturbation label, the concept \(c\in\cC = \bbR^k\) to be the perturbation-relevant latent state, and the feature \(x\in\cX\) to be the observed measurement.
The deterministic indexing kernel is \(\bQ=\Push_f\), with \(f(a)=W(a-a_0)\), and the shared concept-modulation kernel is \(\bB(dc\mid\modval)=\cN(\modval,I)\), yielding the centered attribute potential to be \(\Delta_{a,a_0}^{\bar{\bQ}}(c)-\Delta_{a,a_0}^{\bar{\bQ}}(c_0)=\inner{a-a_0}{W^\top(c-c_0)}\).
Using this quantity, \Cref{thm:tq-characterization} followed by model-specific arguments provided by \citet{von2025representation}, we recover the identifiability guarantee of \citet{von2025representation} under the sufficient-diversity condition \(\Span\setbuild{W(a-a_0)}{a\in\cAo}=\bbR^k\): the perturbation effect matrix is identifiable up to an orthogonal transformation.
For extrapolation under the same condition, the centered potential satisfies the conditions of \Cref{thm:extrapolation-affine-indexing} with \(\varphi(a)=a\) and \(D^{\bar\bQ}(c,c_0)=W^\top(c-c_0)\).
\Cref{thm:extrapolation-affine-indexing} justifies extrapolation to any \(\aex-a_0\in\Span\setbuild{a-a_0}{a\in\cAo}\).
\section{Discussion}\label{sec:discussion}

CMMs provide a common language for studying identifiability and extrapolation through attribute potentials.
Their main role is to separate the generic transition-and-contrast step from model-specific rigidity arguments: feature agreement first induces a latent transition, and the remaining question is whether the model class forces this transition to preserve the relevant attribute potentials.

Our current formulation uses attribute potentials as log-density ratios under a common-support assumption.
This covers many exponential-family, soft-intervention, and stochastic hard-intervention examples, but does not cover deterministic do-interventions or other support-changing mechanisms.
Extending the theory to singular or partially overlapping concept laws would broaden its applicability to causal representation learning without changing the central transition-and-contrast viewpoint.

A further direction is to use CMMs as identifiability guidance for weakly supervised and constraint-aware representation learning.
Recent neurosymbolic and concept-based systems seek to learn intermediate symbolic or concept-level representations from indirect supervision, such as labels, answers, captions, or consistency constraints \citep{duan2023deeplogic,daniele2023deep,barbiero2023interpretable,oikarinen2023label}.
These works show that symbolic or concept-level structure can guide representation learning, but their results typically concern optimization, interpretability, or empirical generalization, rather than rigorous identifiability and extrapolation guarantees.
Our framework suggests one route toward such guarantees by specifying how tasks or attributes modulate latent concept laws and characterizing the induced attribute potentials.

\section*{Acknowledgements}
This research was developed with funding from the Defense Advanced Research Projects Agency (DARPA) via HR0011-25-3-0239, FA8750-23-2-1015, ONR via N00014-23-1-2368, and NSF via IIS-1909816.

\bibliographystyle{plainnat}
\bibliography{bib_refined}


\clearpage

\crefalias{section}{appendix}
\crefalias{subsection}{appendix}
{
\hypersetup{linkcolor=blue}
\renewcommand{\contentsname}{Contents of Appendix}
\tableofcontents
}
\addtocontents{toc}
{\protect\setcounter{tocdepth}{2}}
\newpage
\appendix
\thispagestyle{empty}

\section{Related works}
\label{appendix:related-works}

\subsection{Identifiability in latent-variable representation learning}
Identifiability in latent-variable representation learning asks when latent structure is determined by observed data, up to an allowable ambiguity class such as permutation, component-wise transformations, or affine transformations.
A canonical example is nonlinear ICA, where statistically independent latent components are observed only through a nonlinear mixing map.
Without additional structure, nonlinear ICA and unsupervised disentanglement are generally unidentifiable \citep{locatello2019challenging}.
A central way to restore identifiability is to use auxiliary or conditioning information.
\citet{hyvarinen2019nonlinear} show that an observed auxiliary variable can identify nonlinear ICA representations when latent components are conditionally independent and sufficiently modulated by the auxiliary variable.
\citet{khemakhem2020variational} develop this principle in a variational latent-variable framework, using a condition-dependent exponential-family prior and an injective decoder to obtain identifiability.
\citet{khemakhem2020ice} instantiate the same idea in conditional energy-based models, where structured conditioning of the energy function yields identifiable representations.
A complementary view derives identifiability from restrictions on latent mechanisms and their symmetries rather than from a specific exponential-family form \citep{ahuja2022properties}.
Our framework follows this broader identifiability perspective, but studies conditional generative models indexed by attributes and asks both which latent concept structure is identifiable from observed attributes and when this structure extrapolates to unseen attributes.

\subsection{Causal representation learning}
Causal representation learning (CRL) seeks to recover latent causal variables, and sometimes their causal graph, from high-dimensional observations \citep{Scholkopf2021causal}.
Interventional CRL uses environments or interventions as the source of variation, but the broader literature also exploits temporal structure, paired intervention data, distribution shifts, multi-view observations, and partial observability.
In the static interventional setting, \citet{squires2023linear} study linear latent causal models under injective linear mixing and show that one perfect intervention per latent variable is sufficient, and in a worst-case sense necessary, for identifying the latent causal model and the mixing map.
\citet{ahuja2023interventional} study interventional CRL under perfect and imperfect interventions, showing how intervention-induced geometric structure can identify latent causal variables up to appropriate ambiguities.
\citet{Buchholz2023learning} allow general nonlinear mixing while retaining a linear-Gaussian latent causal model, and obtain identifiability from unpaired interventional data with unknown single-node intervention targets.
\citet{von2023nonparametric} move to nonparametric latent causal models and diffeomorphic mixing under unknown interventions, clarifying both positive identifiability results and residual equivalence classes.
Other work relaxes or changes the intervention model, including soft interventions \citep{zhang2023identifiability}, uncoupled hard interventions and general nonparametric models \citep{varici2024general}, unknown multi-node interventions under linear transformations \citep{varici2024linear}, and score-based identification under linear and general nonlinear transformations \citep{varici2025score}.
Complementary CRL results use weaker or different supervision signals.
\citet{brehmer2022weakly} identify latent causal variables and structure from paired pre- and post-intervention observations, generated from a shared exogenous noise, with random unknown interventions.
\citet{lippe2022citris} use temporal sequences with intervention targets to identify causal factors, while \citet{lippe2023biscuit} replaces observed targets with binary interaction variables.
Multi-view and multi-distribution approaches show that explicit intervention labels are not the only route to identifiability, with guarantees under partial observability \citep{yao2024multiview,xu2024sparsity} and general distribution shifts \citep{zhang2024causal}.
Recent work also studies intrinsic ambiguity and finite-sample behavior in interventional CRL \citep{jin2024learning,acarturk2024sample}.
These works provide model-specific identifiability guarantees under particular sources of variation, whereas CMMs isolate a generic lifting step from observed feature agreement to constrained latent concept transitions and leave model-specific rigidity to separate assumptions.

\subsection{Concept extraction}
Our use of the term ``concept'' is related to work on extracting latent concepts or factors from observed representations, but our goal is different.
Concept-extraction methods in interpretability often aim to find human-facing features in neural activations using examples, labels, sparsity, or dictionary-learning objectives \citep{kim2018interpretability,oikarinen2023label,bricken2023monosemanticity,gao2024scaling}.
Recent theoretical work instead asks when such latent concepts are identifiable.
\citet{cui2025theoretical} analyze identifiability limits for sparse autoencoders, \citet{Zheng2025nonparametric} give nonparametric guarantees for identifying latent concepts, and \citet{squires2026unifying} provide a transition-based framework for unsupervised concept extraction.
\citet{Rajendran2024from} connect CRL and concept-based representation learning by relaxing recovery of true causal variables toward recovery of task-relevant concepts represented in latent space.
\citet{lachapelle22disentangle} show that mechanism sparsity can also drive disentanglement in nonlinear ICA, illustrating how sparse modulation of latent components can reduce non-identifiability.
CMMs differ from these works by making the concept law attribute-indexed: attributes modulate latent concept distributions, so the central questions are not only concept identifiability from observed attributes but also extrapolation of concept and feature laws to unseen attributes.

\subsection{Extrapolation}
A growing line of work studies when models learned from observed conditions can predict behavior under unseen interventions, perturbations, or attribute values.
\citet{lachapelle2023additive} study additive decoders and show that additive latent structure can support Cartesian-product extrapolation by recombining observed factors of variation.
\citet{agarwal2023synthetic} study combinatorial interventions from a causal-inference perspective, using low-rank and Fourier-sparsity structure to identify potential outcomes for unseen intervention combinations.
\citet{du2024compositional} discuss compositional generative modeling more broadly, emphasizing how modular structure can support generalization to unseen combinations.
In latent-variable settings, \citet{saengkyongam2023identifying} establish affine identifiability of latent representations under linear intervention effects and use this to certify extrapolation to out-of-support intervention values for downstream outcomes.
\citet{von2025representation} study distributional perturbation extrapolation under latent mean-shift models and prove representation-identifiability and perturbation-extrapolation guarantees.
\citet{kong22partial} study domain adaptation under partial disentanglement, where target-domain observations are available and the goal is prediction rather than identification of a full unseen interventional distribution.
These works give extrapolation guarantees in specific structural regimes.
CMMs instead formulate extrapolation as an attribute-indexed consistency question: agreement on observed conditional feature laws extrapolates to an unseen attribute exactly when the transported attribute-potential identities extend to that attribute.

\subsection{Unifying frameworks}
Several works abstract away from individual identifiability theorems to identify common proof principles.
\citet{khemakhem2020variational} unify nonlinear ICA and identifiable VAEs through condition-dependent exponential-family latent priors, but this unification is tied to conditional factorization and sufficient variability.
\citet{ahuja2022properties} give a mechanism-based and equivariance view of identifiable representation learning, characterizing residual ambiguity through symmetries shared by the latent mechanisms.
\citet{YaoRancatiCadeiFumeroLocatello2025unifying} propose an invariance-principle view of CRL, showing how several CRL guarantees can be understood through representation alignment with environment-induced invariances.
\citet{Reizinger2025identifiable} introduce identifiable exchangeable mechanisms as a probabilistic framework connecting ICA, CRL, and causal structure learning through exchangeable non-i.i.d. data.
\citet{squires2026unifying} develop a unified framework for unsupervised concept extraction, where identifiability reduces to showing that the intersection of valid transition sets lies inside the allowable ambiguity class.
These frameworks clarify identifiability within particular structural regimes.
CMMs are complementary: they provide a conditional framework in which observed feature agreement induces constrained latent concept transitions, and the same attribute-potential constraints characterize extrapolation to unseen attributes.
\clearpage

\section{Preliminaries}\label{appendix:preliminaries}

\subsection{Notation summary}
\label{ss:notation-summary}

We provide in \Cref{tab:notation-summary} a summary of notation used throughout the appendix.

\begin{table*}[h]
\centering
\scriptsize
\setlength{\tabcolsep}{4pt}
\renewcommand{\arraystretch}{1.15}
\resizebox{\textwidth}{!}{
\begin{tabular}{ll}
\toprule\toprule
\textbf{Symbol} & \textbf{Clarification} \\
\midrule

\multicolumn{2}{l}{\textit{Spaces, attributes, and variables}} \\
\midrule
\(\cA,\cAo,\cAex\)
& Full attribute space, observed attributes, and target/extrapolation attributes.
\\

\(A\)
& Attribute random variable.
\\

\((a,a_0,\aex)\)
& Generic attribute value, anchor attribute in \(\cAo\), and generic extrapolation attribute.
\\

\((\modclass,\modrv,\modval)\)
& Modulator space, modulator random variable, and modulator value.
\\

\((\cC,C,c)\)
& Concept space, concept random variable, and concept value.
\\

\((\cX,X,x)\)
& Feature space, feature random variable, and feature value.
\\

\midrule
\multicolumn{2}{l}{\textit{CMMs and induced kernels}} \\
\midrule
\(\M,\M'\)
& Two CMMs, typically \(\M=(\bQ,\bB,\bK)\) and \(\M'=(\bQ',\bB,\bK')\).
\\

\(\cM=\cQ\times\{\bB\}\times\cK\)
& CMM class with fixed concept-modulation kernel \(\bB\).
\\

\((\cQ,\cK)\)
& Indexing-kernel class and mixing-kernel class.
\\

\((\bQ,\bB,\bK)\)
& Indexing kernel, fixed concept-modulation kernel, and mixing kernel.
\\

\(\bP^\M\)
& Induced feature-generation kernel \(\bK\bB\bQ\).
\\

\((\bar\bQ,\bar\cQ)\)
& Induced concept kernel \(\bB\bQ\) and induced concept-kernel class \(\bB\cQ\).
\\

\(\Push_f\)
& Pushforward operator induced by a measurable map \(f\).
\\

\midrule
\multicolumn{2}{l}{\textit{Reference measures, densities, and potentials}} \\
\midrule
\(\refm\)
& Reference measure on \(\cC\).
\\

\(p_{\bar\bQ}(c\mid a)\)
& Density of \(\bar\bQ(\cdot\mid a)\) with respect to \(\refm\).
\\

\(\Delta_{a,a'}^{\bar\bQ}\)
& Attribute potential, $\Delta_{a,a'}^{\bar\bQ} \defeq \log p_{\bar\bQ}(c \mid a) - \log p_{\bar\bQ}(c \mid a')$.
\\

\(\widetilde\Delta_{a,a_0}^{\bar\bQ}(c,c_0)\)
& Centered attribute potential \(\Delta_{a,a_0}^{\bar\bQ}(c)-\Delta_{a,a_0}^{\bar\bQ}(c_0)\).
\\

\(\refmeq\)
& Equality of kernels on all \(\refm\)-dominated input measures.
\\

\midrule
\multicolumn{2}{l}{\textit{Transitions and equivalence relations}} \\
\midrule
\(\Aut_{\refm}(\cC)\)
& Measurable concept-space automorphisms preserving \(\refm\)-null sets in both directions.
\\

\(\tau\)
& Latent concept transition in \(\Aut_{\refm}(\cC)\).
\\

\(J\tau\)
& Jacobian of the latent concept transition $\tau$.
\\

\(r_{\tau^{-1}}\)
& Radon--Nikodym derivative \(d(\Push_{\tau^{-1}}\refm)/d\refm\).
\\

\((\cT_{\bK}^{\refm},\cT_{\bar\bQ}^{\refm})\)
& Valid mixing and concept transition sets.
\\

\((\frakG,\sim_{\frakG})\)
& Allowed transition group and induced model-equivalence relation.
\\

\midrule
\multicolumn{2}{l}{\textit{Extrapolation notation}} \\
\midrule
\(\varphi\)
& Fixed attribute representation used in affine attribute-potential extrapolation.
\\

\(D^{\bar\bQ}\)
& Coefficient map for centered attribute-potential differences in \Cref{thm:extrapolation-affine-indexing}.
\\

\bottomrule\bottomrule
\end{tabular}
}
\captionsetup{skip=3pt}
\caption{Notation summary.}
\label{tab:notation-summary}
\end{table*}

\subsection{Measure-theoretic background}
\label{ss:measure-theoretic-background}

We use the measure-theoretic notation introduced in the main text.
The only additional convention needed in the appendix concerns restrictions of Markov kernels.
For a measurable space \(\cU\), write \(\cB(\cU)\) for its \(\sigma\)-algebra.
If \(\cU'\in\cB(\cU)\), then \(\cU'\) is equipped with the trace \(\sigma\)-algebra
\[
    \cB(\cU')\defeq
    \setbuild{E\cap \cU'}{E\in\cB(\cU)}.
\]
If \(\bK\in\MK(\cU\to\cV)\), then \(\bK_{\cU'}\in\MK(\cU'\to\cV)\) denotes the domain restriction of \(\bK\), defined by
\[
    \bK_{\cU'}(B\mid u)\defeq \bK(B\mid u),
    \qquad u\in\cU',
    \qquad B\in\cB(\cV).
\]
This is well defined because, for each \(B\in\cB(\cV)\), the map \(u\mapsto \bK(B\mid u)\) is \(\cB(\cU)\)-measurable, and its restriction to \(\cU'\) is \(\cB(\cU')\)-measurable.
Thus \(\bK_{\cU'}\) restricts only the source measurable space; it does not restrict the target space and involves no renormalization.

\section{Proofs}\label{appendix:proofs}

\subsection{Proof of \texorpdfstring{$\sim_\frakG$}{sim-G} being equivalence relation}
\label{appendix:proof-equiv}

\begin{proposition}[$\sim_\frakG$ is an equivalence relation]
\label{prop:sim-G-equivalence}
Fix a CMM class \(\cM=\cQ\times\{\bB\}\times\cK\), an observed attribute set \(\cAo\subseteq\cA\), and a transition group \(\frakG\subseteq\Aut_{\refm}(\cC)\).  
For \(\M=(\bQ,\bB,\bK)\) and \(\M'=(\bQ',\bB,\bK')\), define \(\M\sim_{\frakG}\M'\) on \(\cAo\) if there exists \(\tau\in\frakG\) such that
\[
    \bK'\refmeq\bK\Push_\tau,
    \qquad
    \bar\bQ(\cdot\mid a)=\Push_\tau\bar\bQ'(\cdot\mid a)
    \quad\forall a\in\cAo,
\]
where \(\bar\bQ=\bB\bQ\) and \(\bar\bQ'=\bB\bQ'\).  
Then \(\sim_{\frakG}\) is an equivalence relation on \(\cM\) restricted to \(\cAo\).
\end{proposition}

\begin{proof}
Fix a CMM class \(\cM=\cQ\times\{\bB\}\times\cK\), an observed attribute set \(\cAo\subseteq\cA\), and a transition group \(\frakG\subseteq\Aut_{\refm}(\cC)\).  
Recall that \(\M=(\bQ,\bB,\bK)\) and \(\M'=(\bQ',\bB,\bK')\) satisfy \(\M\sim_{\frakG}\M'\) on \(\cAo\) if there exists \(\tau\in\frakG\) such that
\[
    \bK'\refmeq\bK\Push_\tau,
    \qquad
    \bar\bQ(\cdot\mid a)=\Push_\tau\bar\bQ'(\cdot\mid a)
    \quad\forall a\in\cAo,
\]
where \(\bar\bQ=\bB\bQ\) and \(\bar\bQ'=\bB\bQ'\).

We first prove reflexivity.  
Since \(\frakG\) is a group, \(\id_{\cC}\in\frakG\).  
For every \(\M=(\bQ,\bB,\bK)\), we have \(\bK\refmeq\bK\Push_{\id_{\cC}}\) and \(\bar\bQ(\cdot\mid a)=\Push_{\id_{\cC}}\bar\bQ(\cdot\mid a)\) for all \(a\in\cAo\).  
Hence \(\M\sim_{\frakG}\M\).

We next prove symmetry.  
Suppose \(\M\sim_{\frakG}\M'\) through \(\tau\in\frakG\).  
Then \(\tau^{-1}\in\frakG\).  
For any \(\mu\ll\refm\), since \(\tau^{-1}\in\Aut_{\refm}(\cC)\), we have \(\Push_{\tau^{-1}}\mu\ll\refm\), and therefore
\[
    \bK'\Push_{\tau^{-1}}\mu
    =
    \bK\Push_\tau\Push_{\tau^{-1}}\mu
    =
    \bK\mu.
\]
Thus \(\bK\refmeq\bK'\Push_{\tau^{-1}}\).  
Similarly, for every \(a\in\cAo\),
\[
    \Push_{\tau^{-1}}\bar\bQ(\cdot\mid a)
    =
    \Push_{\tau^{-1}}\Push_\tau\bar\bQ'(\cdot\mid a)
    =
    \bar\bQ'(\cdot\mid a),
\]
with equalities understood modulo \(\refm\)-null sets.  
Hence \(\M'\sim_{\frakG}\M\).

Finally, we prove transitivity.  
Suppose \(\M\sim_{\frakG}\M'\) through \(\tau\in\frakG\), and \(\M'\sim_{\frakG}\M''\) through \(\sigma\in\frakG\).  
Write \(\M''=(\bQ'',\bB,\bK'')\) and \(\bar\bQ''=\bB\bQ''\).  
Since \(\frakG\) is a group, \(\tau\circ\sigma\in\frakG\).  
For every \(\mu\ll\refm\), \(\Push_\sigma\mu\ll\refm\), so
\[
    \bK''\mu
    =
    \bK'\Push_\sigma\mu
    =
    \bK\Push_\tau\Push_\sigma\mu
    =
    \bK\Push_{\tau\circ\sigma}\mu.
\]
Hence \(\bK''\refmeq\bK\Push_{\tau\circ\sigma}\).  
Moreover, for every \(a\in\cAo\),
\[
    \bar\bQ(\cdot\mid a)
    =
    \Push_\tau\bar\bQ'(\cdot\mid a)
    =
    \Push_\tau\Push_\sigma\bar\bQ''(\cdot\mid a)
    =
    \Push_{\tau\circ\sigma}\bar\bQ''(\cdot\mid a).
\]
Thus \(\M\sim_{\frakG}\M''\).  
Therefore \(\sim_{\frakG}\) is reflexive, symmetric, and transitive, and hence is an equivalence relation on \(\cM\) restricted to \(\cAo\).
\end{proof}

\subsection{Proof of \texorpdfstring{\Cref{thm:intersection-individual}}{Thm.~1}}
\label{appendix:proof-intersection}

\Cref{thm:intersection-individual} actually only requires the following condition, which is implied by \Cref{def:ref-support} when $\cAo \neq \varnothing$.
\begin{assumption}\label{def:ref-support-weaker}
    The conditional concept distributions induced by the CMM class \(\cM=\cQ\times\{\bB\}\times\cK\) are dominated by the reference measure $\refm$, i.e.,
    \[
        \bB\bQ(\cdot\mid a)\ll\refm
        \qquad
        \forall \bQ\in\cQ,\ \forall a \in \cA.
    \]
    In addition, there exists a probability measure \(\pi\) on \(\cAo\) such that
    \[
        \int_{\cAo}\bB\bQ(\cdot\mid a)\,\pi(da)
        \sim \refm
        \qquad
        \forall \bQ\in\cQ,
    \]
    so that the observed attributes collectively cover the reference measure \(\refm\) up to its null sets.
\end{assumption}

We now prove the following theorem, restated for convenience, under \Cref{def:ref-support-weaker}. 
\intersection*

\begin{proof}
    Let \(\M'=(\bQ',\bB,\bK')\in\cM\) be feature-equivalent to
    \(\M=(\bQ,\bB,\bK)\) on \(\cAo\). Write
    \[
        \mu_a\defeq\bB\bQ(\cdot\mid a),
        \qquad
        \mu'_a\defeq\bB\bQ'(\cdot\mid a),
    \]
    which satisfies \(\mu_a,\mu'_a\ll\refm\) for all \(a\in\cAo\) by \Cref{def:ref-support-weaker}.
    By \(\refm\)-Blackwell reducibility, choose bimeasurable embeddings \(g,g'\in\cG\) such that
    \[
        \bK\refmeq\widetilde{\bK}\Push_g,
        \qquad
        \bK'\refmeq\widetilde{\bK}\Push_{g'}.
    \]
    Feature equivalence and carrier equality give, for every \(a\in\cAo\),
    \[
        \widetilde{\bK}\Push_g\mu_a
        =
        \widetilde{\bK}\Push_{g'}\mu'_a.
    \]
    Since \(\widetilde{\bK}\) is injective on the relevant pushed-forward carrier,
    \[
        \Push_g\mu_a=\Push_{g'}\mu'_a
        \qquad
        \forall a\in\cAo.
    \]
    Let \(\pi\) be the full-carrier anchor from \Cref{def:ref-support-weaker}, and define
    \[
        \overline\mu_{\bQ}\defeq\int_{\cAo}\mu_a\,\pi(da),
        \qquad
        \overline\mu_{\bQ'}\defeq\int_{\cAo}\mu'_a\,\pi(da).
    \]
    Integrating the previous display over \(a\sim\pi\) gives
    \[
        \Push_g\overline\mu_{\bQ}=\Push_{g'}\overline\mu_{\bQ'}.
    \]
    Since \(\overline\mu_{\bQ}\sim\refm\) and
    \(\overline\mu_{\bQ'}\sim\refm\), it follows that
    \[
        \Push_g\refm\sim\Push_{g'}\refm.
    \]
    Hence \(g(\cC)\) and \(g'(\cC)\) agree modulo the corresponding pushed-forward carrier:
    \[
        \refm\{c:g'(c)\notin g(\cC)\}=0,
        \qquad
        \refm\{c:g(c)\notin g'(\cC)\}=0.
    \]
    Therefore the maps
    \[
        \tau\defeq g^{-1}\circ g',
        \qquad
        \tau^{-1}\defeq g'^{-1}\circ g
    \]
    are well-defined modulo \(\refm\)-null sets. Moreover, since
    \(\Push_g\Push_\tau\refm=\Push_{g'}\refm\sim\Push_g\refm\), injectivity of \(g\) implies
    \(\Push_\tau\refm\sim\refm\). The same argument with \(g\) and \(g'\) exchanged gives
    \(\Push_{\tau^{-1}}\refm\sim\refm\). Thus
    \[
        \tau\in\Aut_{\refm}(\cC),
        \qquad
        g'=g\circ\tau
        \quad \refm\text{-a.e.}
    \]

    Since \(\mu'_a\ll\refm\), the relation \(g'=g\circ\tau\) \(\refm\)-a.e. implies
    \[
        \Push_{g'}\mu'_a=\Push_g\Push_\tau\mu'_a.
    \]
    Combining this with \(\Push_g\mu_a=\Push_{g'}\mu'_a\), we obtain
    \[
        \Push_g\mu_a=\Push_g\Push_\tau\mu'_a.
    \]
    Since \(g\) is a bimeasurable embedding, \(\Push_g\) is injective on measures on \(\cC\). Hence
    \[
        \mu_a=\Push_\tau\mu'_a
        \qquad
        \forall a\in\cAo.
    \]
    Equivalently,
    \[
        \bB\bQ(\cdot\mid a)=\Push_\tau\bB\bQ'(\cdot\mid a)
        \qquad
        \forall a\in\cAo,
    \]
    so \(\tau\in\cT_{\bar\bQ}^{\refm}(\bar\cQ;\cAo)\).

    Finally, for every \(\mu\ll\refm\),
    \[
        \bK'\mu
        =\widetilde{\bK}\Push_{g'}\mu
        =\widetilde{\bK}\Push_g\Push_\tau\mu
        =\bK\Push_\tau\mu.
    \]
    Thus \(\bK'\refmeq\bK\Push_\tau\), so
    \(\tau\in\cT_{\bK}^{\refm}(\cK)\). Therefore
    \[
        \tau\in\cT_{\bK}^{\refm}(\cK)\cap\cT_{\bar\bQ}^{\refm}(\bar\cQ;\cAo)
        \subseteq\frakG.
    \]
    By \Cref{def:blackwell-equivalence}, \(\M\sim_\frakG\M'\). Since \(\M'\) was arbitrary among
    feature-equivalent alternatives, \(\M\) is identifiable up to \(\sim_\frakG\).
\end{proof}

\subsection{Proof of \texorpdfstring{\Cref{thm:tq-characterization}}{Thm.~2}}

\tqcharacterization*

\label{appendix:proof-tq-characterization}
\begin{proof}
We first show that (i) and (ii) are equivalent.  
For any \(a\in\cA'\), the identity \(\bar\bQ(\cdot\mid a)=\Push_\tau\bar\bQ'(\cdot\mid a)\) is equivalent, since \(\tau\in\Aut_{\refm}(\cC)\), to
\[
    \bar\bQ'(\cdot\mid a)=\Push_{\tau^{-1}}\bar\bQ(\cdot\mid a).
\]
If \(\bar\bQ(\cdot\mid a)\) has density \(p_{\bar\bQ}(\cdot\mid a)\) with respect to \(\refm\), then the Radon--Nikodym chain rule for \(\Push_{\tau^{-1}}\) gives
\[
    \frac{d\,\Push_{\tau^{-1}}\bar\bQ(\cdot\mid a)}{d\refm}(c)
    =
    p_{\bar\bQ}(\tau(c)\mid a) r_{\tau^{-1}}(c),
    \qquad
    r_{\tau^{-1}}=\frac{d(\Push_{\tau^{-1}}\refm)}{d\refm}.
\]
Thus \(\bar\bQ'(\cdot\mid a)=\Push_{\tau^{-1}}\bar\bQ(\cdot\mid a)\) holds for every \(a\in\cA'\) if and only if
\[
    p_{\bar\bQ'}(c\mid a)
    =
    p_{\bar\bQ}(\tau(c)\mid a) r_{\tau^{-1}}(c)
\]
holds for every \(a\in\cA'\) and \(\refm\)-a.e. \(c\), which proves the equivalence between (i) and (ii).

We next show that (ii) implies (iii).  
Taking \(a=a_0\) in (ii) gives the anchored density identity
\[
    p_{\bar\bQ'}(c\mid a_0)
    =
    p_{\bar\bQ}(\tau(c)\mid a_0) r_{\tau^{-1}}(c)
\]
for \(\refm\)-a.e. \(c\).  
Now fix any \(a\in\cA'\).  
By the common-support assumption, \(p_{\bar\bQ'}(\cdot\mid a)\), \(p_{\bar\bQ'}(\cdot\mid a_0)\), \(p_{\bar\bQ}(\cdot\mid a)\), and \(p_{\bar\bQ}(\cdot\mid a_0)\) are positive and finite \(\refm\)-a.e.; since \(\tau\in\Aut_{\refm}(\cC)\), the same is true of \(p_{\bar\bQ}(\tau(\cdot)\mid a)\) and \(p_{\bar\bQ}(\tau(\cdot)\mid a_0)\) on a \(\refm\)-full set.  
Intersecting this full-measure set with the full-measure sets on which (ii) holds for \(a\) and \(a_0\), we may take logarithms and subtract:
\[
\begin{aligned}
    \log p_{\bar\bQ'}(c\mid a)-\log p_{\bar\bQ'}(c\mid a_0)
    &=
    \left[\log p_{\bar\bQ}(\tau(c)\mid a)+\log r_{\tau^{-1}}(c)\right] \\
    &\quad -
    \left[\log p_{\bar\bQ}(\tau(c)\mid a_0)+\log r_{\tau^{-1}}(c)\right] \\
    &=
    \log p_{\bar\bQ}(\tau(c)\mid a)-\log p_{\bar\bQ}(\tau(c)\mid a_0).
\end{aligned}
\]
Therefore
\[
    \Delta_{a,a_0}^{\bar\bQ'}(c)
    =
    \Delta_{a,a_0}^{\bar\bQ}(\tau(c))
\]
for \(\refm\)-a.e. \(c\).  
Since \(a\in\cA'\) was arbitrary, the attribute-potential identities hold for every \(a\in\cA'\), so (iii) follows.

Conversely, suppose (iii) holds.  
Then for every \(a\in\cA'\),
\[
\begin{aligned}
    p_{\bar\bQ'}(c\mid a)
    &=
    p_{\bar\bQ'}(c\mid a_0)\exp\{\Delta_{a,a_0}^{\bar\bQ'}(c)\} \\
    &=
    p_{\bar\bQ}(\tau(c)\mid a_0)r_{\tau^{-1}}(c)\exp\{\Delta_{a,a_0}^{\bar\bQ}(\tau(c))\} \\
    &=
    p_{\bar\bQ}(\tau(c)\mid a)r_{\tau^{-1}}(c)
\end{aligned}
\]
for \(\refm\)-a.e. \(c\), so (ii) holds.  
The final statement follows from the definition of \(\cT_{\bar\bQ}^\refm(\bar\cQ;\cAo)\).
\end{proof}

\subsection{Proof of \texorpdfstring{\Cref{thm:extrapolation-delta-iff}}{Thm.~3}}
\label{appendix:proof-extrapolation-delta-iff}

\extrapolationdeltaiff*

\begin{proof}
Write
\[
    \bar\bQ=\bB\bQ,
    \qquad
    \bar\bQ'=\bB\bQ',
    \qquad
    \mu_a\defeq\bar\bQ(\cdot\mid a),
    \qquad
    \mu'_a\defeq\bar\bQ'(\cdot\mid a).
\]
By the common \(\refm\)-support assumption, \(\mu_a,\mu'_a\ll\refm\) for every \(a\in\cA\).  
Since \(\tau\in\Aut_{\refm}(\cC)\), we also have \(\Push_\tau\mu'_a\ll\refm\) for every \(a\in\cA\).

The transition \(\tau\) induced by \Cref{thm:intersection-individual} satisfies
\[
    \bK'\refmeq\bK\Push_\tau,
    \qquad
    \mu_a=\Push_\tau\mu'_a
    \quad\forall a\in\cAo.
\]
In particular,
\begin{equation}\label{eq:proof-edi-1}
    \mu_{a_0}=\Push_\tau\mu'_{a_0}.
\end{equation}

We first show that, for each \(a\in\cAex\),
\begin{equation}\label{eq:proof-edi-2}
    \bP^\M(\cdot\mid a)=\bP^{\M'}(\cdot\mid a)
    \quad\Longleftrightarrow\quad
    \mu_a=\Push_\tau\mu'_a.
\end{equation}
Indeed, since \(\mu'_a\ll\refm\) and \(\bK'\refmeq\bK\Push_\tau\),
\[
    \bP^{\M'}(\cdot\mid a)
    =
    \bK'\mu'_a
    =
    \bK\Push_\tau\mu'_a.
\]
Also \(\bP^\M(\cdot\mid a)=\bK\mu_a\).  
Thus feature equality at \(a\) is equivalent to
\[
    \bK\mu_a=\bK\Push_\tau\mu'_a.
\]
It remains only to justify that \(\bK\) is injective on measures dominated by \(\refm\).  
By \(\refm\)-Blackwell reducibility, choose \(g\in\cG\) and \(\widetilde{\bK}\) such that
\[
    \bK\refmeq\widetilde{\bK}\Push_g.
\]
If \(\alpha,\beta\ll\refm\) and \(\bK\alpha=\bK\beta\), then
\[
    \widetilde{\bK}\Push_g\alpha
    =
    \widetilde{\bK}\Push_g\beta.
\]
By the injectivity condition in \Cref{def:mu-blackwell-reducible},
\[
    \Push_g\alpha=\Push_g\beta.
\]
Since \(g\) is a bimeasurable embedding, \(\Push_g\) is injective on probability measures on \(\cC\), hence \(\alpha=\beta\).  
Applying this with \(\alpha=\mu_a\) and \(\beta=\Push_\tau\mu'_a\) proves \eqref{eq:proof-edi-2}.

Taking \eqref{eq:proof-edi-2} for all \(a\in\cAex\) gives
\begin{equation}\label{eq:proof-edi-3}
    \bP^\M_{\cAex}=\bP^{\M'}_{\cAex}
    \quad\Longleftrightarrow\quad
    \bar\bQ_{\cAex}=\Push_\tau\bar\bQ'_{\cAex}.
\end{equation}

We now translate the right-hand side of \eqref{eq:proof-edi-3} into attribute potentials.  
Apply \Cref{thm:tq-characterization} with \(\cA'=\cAex\).  
It states that
\[
    \bar\bQ_{\cAex}=\Push_\tau\bar\bQ'_{\cAex}
\]
is equivalent to the conjunction of the anchored density identity
\begin{equation}\label{eq:proof-edi-4}
    p_{\bar\bQ'}(c\mid a_0)
    =
    p_{\bar\bQ}(\tau(c)\mid a_0)r_{\tau^{-1}}(c)
    \quad
    \refm\text{-a.e. }c
\end{equation}
and the transported attribute-potential identities
\begin{equation}\label{eq:proof-edi-5}
    \Delta_{a,a_0}^{\bar\bQ'}(c)
    =
    \Delta_{a,a_0}^{\bar\bQ}(\tau(c))
    \quad
    \text{for every }a\in\cAex,\ \refm\text{-a.e. }c.
\end{equation}
However, \eqref{eq:proof-edi-4} already follows from \eqref{eq:proof-edi-1} by the density form of \Cref{thm:tq-characterization}, since \(a_0\in\cAo\).  
Therefore, for the fixed transition \(\tau\) induced from the observed attributes, the condition
\[
    \bar\bQ_{\cAex}=\Push_\tau\bar\bQ'_{\cAex}
\]
is equivalent to \eqref{eq:proof-edi-5} alone.

Combining this equivalence with \eqref{eq:proof-edi-3}, we obtain
\[
    \bP^\M_{\cAex}=\bP^{\M'}_{\cAex}
    \quad\Longleftrightarrow\quad
    \Delta_{a,a_0}^{\bar\bQ'}(c)
    =
    \Delta_{a,a_0}^{\bar\bQ}(\tau(c))
    \text{ for every }a\in\cAex,\ \refm\text{-a.e. }c.
\]
This proves the claimed equivalence.

The final statement follows immediately: if the transported attribute-potential identities that hold on \(\cAo\) extend to all \(a\in\cAex\), then the right-hand side above holds, and therefore \(\bP^\M_{\cAex}=\bP^{\M'}_{\cAex}\).
\end{proof}

\subsection{Proof of \texorpdfstring{\Cref{thm:extrapolation-affine-indexing}}{Thm.~4}}
\label{appendix:proof-extrapolation-affine-indexing}

\affineindexingextrapolation*

\begin{proof}
Fix \(\aex\in\cAex\).
Since \(\varphi(\cAex)\) is contained in the affine hull of \(\varphi(\cAo)\), there exist \(a_1,\ldots,a_\ell\in\cAo\) and \(\alpha_1,\ldots,\alpha_\ell\in\bbR\) such that \(\varphi(\aex)=\sum_{i=1}^\ell \alpha_i\varphi(a_i)\) and \(\sum_{i=1}^\ell\alpha_i=1\).
By the assumed affine representation of attribute-potential differences, for \(\M=(\bQ,\bB,\bK)\) we have
\begin{align*}
    \Delta_{\aex,a_0}^{\bar{\bQ}}(c)-\Delta_{\aex,a_0}^{\bar{\bQ}}(c_0)
    &=
    \inner{\varphi(\aex)-\varphi(a_0)}{D^\M(c,c_0)} \\
    &=
    \inner{\sum_{i=1}^\ell \alpha_i\varphi(a_i)-\sum_{i=1}^\ell \alpha_i\varphi(a_0)}{D^\M(c,c_0)} \\
    &=
    \sum_{i=1}^\ell \alpha_i\inner{\varphi(a_i)-\varphi(a_0)}{D^\M(c,c_0)} \\
    &=
    \sum_{i=1}^\ell \alpha_i\left[\Delta_{a_i,a_0}^{\bar{\bQ}}(c)-\Delta_{a_i,a_0}^{\bar{\bQ}}(c_0)\right],
\end{align*}
where the second equality uses \(\sum_{i=1}^\ell\alpha_i=1\).
The same argument for \(\M'=(\bQ',\bB,\bK')\) gives
\begin{align*}
    \Delta_{\aex,a_0}^{\bar{\bQ}'}(c)-\Delta_{\aex,a_0}^{\bar{\bQ}'}(c_0)
    =
    \sum_{i=1}^\ell \alpha_i\left[\Delta_{a_i,a_0}^{\bar{\bQ}'}(c)-\Delta_{a_i,a_0}^{\bar{\bQ}'}(c_0)\right].
\end{align*}
Since \(a_i\in\cAo\), \Cref{thm:tq-characterization} gives the observed transported identities
\[
    \Delta_{a_i,a_0}^{\bar{\bQ}'}(c)
    =
    \Delta_{a_i,a_0}^{\bar{\bQ}}(\tau(c))
    \qquad
    \text{for \(\refm\)-a.e. }c
\]
for each \(i\in[\ell]\).
Subtracting the identity evaluated at \(c_0\) from the identity evaluated at \(c\), we obtain
\[
    \Delta_{a_i,a_0}^{\bar{\bQ}'}(c)-\Delta_{a_i,a_0}^{\bar{\bQ}'}(c_0)
    =
    \Delta_{a_i,a_0}^{\bar{\bQ}}(\tau(c))-\Delta_{a_i,a_0}^{\bar{\bQ}}(\tau(c_0))
\]
for \(\refm \otimes \refm\)-a.e. \((c,c_0)\), given that $\Push_\tau \refm \sim \refm$.
Combining the preceding two displays gives
\begin{align*}
    \Delta_{\aex,a_0}^{\bar{\bQ}'}(c)-\Delta_{\aex,a_0}^{\bar{\bQ}'}(c_0)
    &=
    \sum_{i=1}^\ell \alpha_i\left[\Delta_{a_i,a_0}^{\bar{\bQ}}(\tau(c))-\Delta_{a_i,a_0}^{\bar{\bQ}}(\tau(c_0))\right].
\end{align*}
On the other hand, applying the affine representation for \(\M\) at the pair \((\tau(c),\tau(c_0))\) gives
\begin{align*}
    \Delta_{\aex,a_0}^{\bar{\bQ}}(\tau(c))-\Delta_{\aex,a_0}^{\bar{\bQ}}(\tau(c_0))
    &=
    \sum_{i=1}^\ell \alpha_i\left[\Delta_{a_i,a_0}^{\bar{\bQ}}(\tau(c))-\Delta_{a_i,a_0}^{\bar{\bQ}}(\tau(c_0))\right]
\end{align*}
for \(\refm \otimes \refm\)-a.e. \((c, c_0)\), since \(\tau\in\Aut_{\refm}(\cC)\) preserves \(\refm\)-null sets.
Therefore
\[
    \Delta_{\aex,a_0}^{\bar{\bQ}'}(c)-\Delta_{\aex,a_0}^{\bar{\bQ}'}(c_0)
    =
    \Delta_{\aex,a_0}^{\bar{\bQ}}(\tau(c))-\Delta_{\aex,a_0}^{\bar{\bQ}}(\tau(c_0))
\]
for \(\refm \otimes \refm\)-a.e. \((c,c_0)\).
Define
\[
    H(c)\defeq\Delta_{\aex,a_0}^{\bar{\bQ}'}(c)-\Delta_{\aex,a_0}^{\bar{\bQ}}(\tau(c)).
\]
The preceding display says \(H(c)-H(c_0)=0\) for \(\refm \otimes \refm\)-a.e. \((c,c_0)\).
This implies that, through Fubini, \(H\) is \(\refm\)-a.e. constant, so there exists a scalar \(b(\aex)\) such that
\[
    \Delta_{\aex,a_0}^{\bar{\bQ}'}(c)
    =
    \Delta_{\aex,a_0}^{\bar{\bQ}}(\tau(c))+b(\aex)
\]
for \(\refm\)-a.e. \(c\).
Let \(r_{\tau^{-1}}(c)\defeq d(\Push_{\tau^{-1}}\refm)/d\refm(c)\).
By the anchored density identity in \Cref{thm:tq-characterization},
\[
    p_{\bar{\bQ}'}(c\mid a_0)
    =
    p_{\bar{\bQ}}(\tau(c)\mid a_0)r_{\tau^{-1}}(c)
\]
for \(\refm\)-a.e. \(c\).
Therefore,
\begin{align*}
    p_{\bar{\bQ}'}(c\mid\aex)
    &=
    p_{\bar{\bQ}'}(c\mid a_0)\exp\{\Delta_{\aex,a_0}^{\bar{\bQ}'}(c)\} \\
    &=
    p_{\bar{\bQ}}(\tau(c)\mid a_0)r_{\tau^{-1}}(c)\exp\{\Delta_{\aex,a_0}^{\bar{\bQ}}(\tau(c))+b(\aex)\} \\
    &=
    \exp\{b(\aex)\}p_{\bar{\bQ}}(\tau(c)\mid\aex)r_{\tau^{-1}}(c).
\end{align*}
Integrating both sides with respect to \(\refm\) gives
\begin{align*}
    1
    &=
    \exp\{b(\aex)\}\int p_{\bar{\bQ}}(\tau(c)\mid\aex)r_{\tau^{-1}}(c)\,\refm(dc) \\
    &=
    \exp\{b(\aex)\}\int p_{\bar{\bQ}}(u\mid\aex)\,\refm(du) \\
    &=
    \exp\{b(\aex)\}.
\end{align*}
Hence \(b(\aex)=0\).
Thus \(\Delta_{\aex,a_0}^{\bar{\bQ}'}(c)=\Delta_{\aex,a_0}^{\bar{\bQ}}(\tau(c))\) for \(\refm\)-a.e. \(c\).
Since \(\aex\in\cAex\) was arbitrary, the transported attribute-potential identities extend from \(\cAo\) to \(\cAex\).
\end{proof}

\subsection{Proof of \texorpdfstring{\Cref{cor:interaction-extrapolation}}{Cor.~1}}
\label{appendix:proof-interaction-extrapolation}

\interactionextrapolation*

\begin{proof}
Set \(\cAex=2^{[m]}\) and take \(a_0=\varnothing\).  
For each \(S\subseteq[m]\), define the interaction-incidence vector \(\varphi_\cH(S)\in\bbR^{\cH}\) by \(\varphi_\cH(S)_T=\mathbf{1}\{T\subseteq S\}\) for \(T\in\cH\).  
We first show that the assumed density model has the affine attribute-potential form required by \Cref{thm:extrapolation-affine-indexing}.  
Write
\[
    \ell_S^\M(c)
    =
    \sum_{T\in\cH, T\subseteq S}
    h_T^\M(c),
\]
so that \(p^\M(c\mid S)=\exp\{\ell_S^\M(c)\}/Z_S^\M\) for a normalizing constant \(Z_S^\M\).  
Since \(p^\M(c\mid\varnothing)=\exp\{h_\varnothing^\M(c)\}/Z_\varnothing^\M\), the attribute potential relative to \(\varnothing\) is
\[
    \Delta_{S,\varnothing}^\M(c)
    =
    \log\frac{p^\M(c\mid S)}{p^\M(c\mid\varnothing)}
    =
    \sum_{T\in\cH, T\subseteq S, T \neq \varnothing}
    h_T^\M(c)
    -
    \log\frac{Z_S^\M}{Z_\varnothing^\M}.
\]
The normalizing constant does not depend on \(c\), so it disappears after centering at any fixed \(c_0\in\cC\).  
Thus
\[
    \Delta_{S,\varnothing}^\M(c)-\Delta_{S,\varnothing}^\M(c_0)
    =
    \sum_{T\in\cH, T\subseteq S, T \neq \varnothing}
    \bigl(h_T^\M(c)-h_T^\M(c_0)\bigr).
\]
Define \(D^\M(c,c_0)\in\bbR^{\cH}\) by \(D_\varnothing^\M(c,c_0)=0\) and \(D_T^\M(c,c_0)=h_T^\M(c)-h_T^\M(c_0)\) for \(T\neq\varnothing\).  
Since \(\varphi_\cH(\varnothing)\) is equal to \(1\) in the \(\varnothing\)-coordinate and \(0\) elsewhere, the previous display can be written as
\[
    \Delta_{S,\varnothing}^\M(c)-\Delta_{S,\varnothing}^\M(c_0)
    =
    \inner{\varphi_\cH(S)-\varphi_\cH(\varnothing)}{D^\M(c,c_0)}.
\]
Hence the fixed representation \(\varphi(S)=\varphi_\cH(S)\) satisfies the attribute-potential assumption of \Cref{thm:extrapolation-affine-indexing}.  

It remains only to verify that each target vector \(\varphi_\cH(S)\), for \(S\subseteq[m]\), belongs to the affine hull of the observed vectors \(\setbuild{\varphi_\cH(U)}{U\in\cH}\).
This is a standard consequence of M{\"o}bius inversion, equivalently of the invertibility of the zeta matrix of a finite poset, in the incidence-algebra formulation of \citet{Rota1964Mobius}.
For completeness, we recall the short argument.

We first show linear that the observed vectors \(\setbuild{\varphi_\cH(U)}{U\in\cH}\) are linearly independent.
Suppose \(\sum_{U\in\cH}\alpha_U \varphi_\cH(U)=0\).  
If some coefficient is nonzero, choose an inclusion-maximal \(U_0\in\cH\) among the sets with \(\alpha_{U_0}\neq 0\), which is possible because \(\cH\) is finite.  
Looking at the \(U_0\)-coordinate of the vector equality gives
\[
    0
    =
    \sum_{U\in\cH}\alpha_U \varphi_\cH(U)_{U_0}
    =
    \sum_{\substack{U\in\cH\\U_0\subseteq U}}\alpha_U.
\]
By maximality of \(U_0\), every term in the last sum except \(U=U_0\) has coefficient zero.  
Therefore \(0=\alpha_{U_0}\), contradicting the choice of \(U_0\).  
Thus all coefficients are zero, so the observed vectors are linearly independent.  
There are exactly \(\abs{\cH}\) such vectors in the \(\abs{\cH}\)-dimensional space \(\bbR^{\cH}\), so they form a basis of \(\bbR^{\cH}\).  

Now fix any target \(S\subseteq[m]\).  
Since the observed vectors form a basis, there exist coefficients \(\alpha_U(S)\), \(U\in\cH\), such that
\[
    \varphi_\cH(S)
    =
    \sum_{U\in\cH}\alpha_U(S)\varphi_\cH(U).
\]
This expansion is automatically affine rather than merely linear.  
Indeed, looking at the \(\varnothing\)-coordinate gives
\[
    1
    =
    \varphi_\cH(S)_\varnothing
    =
    \sum_{U\in\cH}\alpha_U(S)\varphi_\cH(U)_\varnothing
    =
    \sum_{U\in\cH}\alpha_U(S),
\]
because \(\varnothing\subseteq S\) and \(\varnothing\subseteq U\) for every \(U\in\cH\).  
Therefore \(\varphi_\cH(S)\) lies in the affine hull of \(\setbuild{\varphi_\cH(U)}{U\in\cH}\).  
Since \(S\subseteq[m]\) was arbitrary, \(\varphi_\cH(\cAex)\) is contained in the affine hull of \(\varphi_\cH(\cAo)\).  

Let \(\M,\M'\in\cM\) be feature-equivalent on \(\cAo=\cH\).  
The preceding two paragraphs verify the hypotheses of \Cref{thm:extrapolation-affine-indexing} with \(\varphi=\varphi_\cH\).  
Therefore the transported attribute-potential identities that hold on the observed attributes \(\cH\) extend to every \(S\subseteq[m]\).  
By \Cref{thm:extrapolation-delta-iff}, these extended attribute-potential identities imply \(\bP_S^\M=\bP_S^{\M'}\) for every \(S\subseteq[m]\).  
Hence \(\M\) and \(\M'\) are feature-equivalent on all of \(2^{[m]}\).
\end{proof}

\section{Deferred details}
\label{appendix:details}

\subsection{Details for the running example}
\label{app:running-example-details}

We provide the linear-algebra details behind
\Cref{ex:running-example-theorems}. Fix
\(\tau \in \cT_{\bar\bQ}^{\refm}(\bar\cQ;\cAo)\). By
\Cref{thm:tq-characterization}, there exists
\(\bQ'=\Push_{f'}\in\cQ\) such that, after centering at any fixed \(c_0\),
\[
    \inner{f'(a)-f'(a_0)}{c-c_0}
    =
    \inner{f(a)-f(a_0)}{\tau(c)-\tau(c_0)}
    \qquad
    \forall a\in\cAo .
\]
Let
\[
    V\defeq\Span\setbuild{f(a)-f(a_0)}{a\in\cAo}.
\]
Choose \(a_1,\ldots,a_r\in\cAo\) such that
\(f(a_i)-f(a_0)\), \(i=1,\ldots,r\), form a basis of \(V\). Define
\[
    F_o\defeq
    (f(a_1)-f(a_0),\ldots,f(a_r)-f(a_0))^\top,
    \qquad
    F'_o\defeq
    (f'(a_1)-f'(a_0),\ldots,f'(a_r)-f'(a_0))^\top .
\]
Let \(V'\) be the row space of \(F'_o\), and let \(\Pi_V\) and \(\Pi_{V'}\) denote
the orthogonal projections onto \(V\) and \(V'\), respectively. Stacking the
preceding identities gives
\[
    F'_o(c-c_0)=F_o(\tau(c)-\tau(c_0)).
\]
Since \(F_o\) has full row rank and row space \(V\),
\[
    \Pi_V=F_o^\top(F_oF_o^\top)^{-1}F_o .
\]
Therefore,
\[
    \Pi_V\tau(c)
    =
    M \Pi_{V'}c+b,
    \qquad
    M\defeq F_o^\top(F_oF_o^\top)^{-1}F'_o,
    \qquad
    b\defeq \Pi_V\tau(c_0)-M\Pi_{V'}c_0 .
\]
This proves the projection identity used in the partial-identifiability
remark.

If \(V=\bbR^k\), then \(F_o\) has full column rank. The stacked identity forces
\(F'_o\) to have full column rank as well; otherwise
\(F_o(\tau(c)-\tau(c_0))\) would lie in a proper subspace for all \(c\),
contradicting \(\tau\in\Aut_{\refm}(\bbR^k)\). Hence
\(V'=\bbR^k\), \(\Pi_V=\Pi_{V'}=I\), and the projection identity reduces to
\[
    \tau(c)=Mc+b .
\]
Thus, under the full-span condition,
\[
    \cT_{\bar\bQ}^{\refm}(\bar\cQ;\cAo)
    \subseteq
    \mathrm{Affine}(k)\cap\Aut_{\refm}(\bbR^k).
\]

\section{Recoveries of representative prior work}
\label{appendix:prior-work-recovery}

The recoveries below are not intended as new proofs of the cited identifiability theorems.
Instead, each recovery is organized as a short CMM dictionary followed by the CMM proof object that the original paper analyzes.
The generic transition step is supplied by \Cref{thm:intersection-individual} and, when a density-level statement is needed, by \Cref{thm:tq-characterization}.
The remaining rigidity step is delegated to the corresponding theorem or proof in the cited work.

\paragraph{Shared architecture of the recoveries.}
Each recovery follows the same template:
\begin{enumerate}[label=(\roman*),leftmargin=*]
    \item \textbf{CMM translation. }
Each recovery first translates the prior model into CMM notation by specifying the attribute \(A\), the modulator \(\modrv\), the concept \(C\), the feature \(X\), and the kernels \(\bQ\), \(\bB\), and \(\bK\).

    \item \textbf{Attribute potentials} or \textbf{Proof object. }
The next paragraph computes the attribute potential or the equivalent object used by the cited proof, such as a centered potential, exponentiated density ratio, score contrast, covariance equation, or quadratic transition identity.

    \item \textbf{Recovery of guarantees. }
The proposition starts from feature-equivalent CMMs, applies the CMM transition theorem, rewrites the transported identity as the cited paper's proof object, and invokes the paper-specific rigidity result that yields the final ambiguity class.
\end{enumerate}
Accordingly, the short proofs below only verify the CMM-to-paper identity displayed in the proposition; they do not reprove the paper-specific rigidity arguments.

\paragraph{Applicability of \Cref{thm:intersection-individual}.}
In the continuous recoveries below, we take \(\refm\) to be Lebesgue measure on the corresponding concept space \(\cC=\bbR^d\) or \(\bbR^n\).  
Except where explicitly noted, the mixing kernels are deterministic kernels \(\Push_g\), where \(g\colon\cC\to\cX\) is an admissible diffeomorphic embedding; in linear CRL, \(g\) is the full-column-rank linear mixing matrix \(G\), which is a bimeasurable embedding of \(\bbR^d\) onto its image.  
Thus the mixing classes are \(\refm\)-Blackwell reducible by taking \(\widetilde\cX=\cX\), \(\widetilde{\bK}=\Push_{\id_{\cX}}\), and \(\cG\) to be the relevant class of admissible embeddings, since \(\xi\mapsto\widetilde{\bK}\xi\) is the identity map on the pushed-forward measures.  
For noisy iVAE-type observation models, we instead write the mixing kernel as \(\bK=\widetilde{\bK}\Push_g\), where \(g\) is the decoder embedding and \(\widetilde{\bK}\) is the shared observation-noise channel; the required Blackwell-reducibility condition is that this shared channel be injective on the relevant pushed-forward latent laws, or equivalently that we work after the usual deconvolution reduction.  

On the concept side, we assume the common-carrier condition required by \Cref{thm:intersection-individual}: the induced concept laws in the model class are dominated by \(\refm\), and their densities are positive and finite \(\refm\)-a.e. on the attributes under consideration.  
This condition holds directly for the Gaussian CRL, score-based CRL, and nonparametric CRL recoveries under their full-support density assumptions, and for nonlinear ICA / iVAE under the usual positivity assumptions on the conditional source or prior densities.  
For linear CRL \citep{squires2023linear}, the original theorem is covariance-based and does not require densities; when we invoke \Cref{thm:intersection-individual}, we are considering the subclass satisfying the additional common-carrier condition, for example when the noise \(\epsilon\) has a strictly positive finite Lebesgue density and all \(B_k\) are invertible.  
Under these standing conditions, any feature-equivalent alternative admits the transition supplied by \Cref{thm:intersection-individual}; the individual recoveries below then identify the corresponding proof object and delegate the model-specific rigidity step to the cited paper.

\subsection{Nonlinear ICA}
\label{appendix:proof-nonlinear-ica}

\paragraph{CMM translation.}
For nonlinear ICA \citep{hyvarinen2019nonlinear}, take \(A\) to be the auxiliary variable, \(\cC=\bbR^k\), \(C=Z\), and \(X=f(Z)\).
For each attribute value \(a\), define the modulator value
\[
    \modval_a
    \defeq
    \bigl(q_1(\cdot,a),\ldots,q_k(\cdot,a)\bigr).
\]
Thus the modulator contains only the coordinate-wise conditional source log-potentials.
The normalizer is the derived functional
\[
    \Gamma(\modval)
    \defeq
    \log\int_{\bbR^k}
    \exp\left\{
        \sum_{i=1}^k \modval_i(c_i)
    \right\}
    dc,
\]
and we write \(\Gamma(a)\defeq\Gamma(\modval_a)\).
The deterministic indexing kernel is
\[
    \bQ=\Push_{a\mapsto\modval_a}.
\]
The shared concept-modulation kernel maps a tuple of source log-potentials to the corresponding conditionally independent source density:
\[
    \bB(dc\mid\modval)
    =
    \exp\left\{
        \sum_{i=1}^k \modval_i(c_i)
        -
        \Gamma(\modval)
    \right\}dc.
\]
Thus, for \(\bar\bQ=\bB\bQ\),
\[
    \log p_{\bar\bQ}(c\mid a)
    =
    \sum_{i=1}^k q_i(c_i,a)-\Gamma(a).
\]
The mixing kernel is deterministic, \(\bK=\Push_f\).

\paragraph{Attribute potentials.}
Relative to an anchor \(a_0\), the attribute potential is
\[
    \Delta_{a,a_0}^{\bar\bQ}(c)
    =
    \sum_{i=1}^k
    \{q_i(c_i,a)-q_i(c_i,a_0)\}
    -
    \{\Gamma(a)-\Gamma(a_0)\}.
\]
The normalizer contrast is independent of \(c\), so all mixed second derivatives of the potential are determined only by the coordinate-wise contrasts
\[
    h_i(u;a,a_0)
    \defeq
    q_i(u,a)-q_i(u,a_0).
\]
This separable potential is the CMM proof object whose transported mixed derivatives reproduce the nonlinear ICA equations of \citet{hyvarinen2019nonlinear}.

\paragraph{Recovery of guarantees.}
\begin{proposition}[CMM recovery of Theorem 1 of \citep{hyvarinen2019nonlinear}]
\label{prop:recovery-nlica}
Let
\[
    \M=(\Push_{a\mapsto\modval_a},\bB,\Push_f),
    \qquad
    \M'=(\Push_{a\mapsto\modval'_a},\bB,\Push_{f'})
\]
be feature-equivalent nonlinear ICA CMMs satisfying the common-support and smoothness assumptions needed for \Cref{thm:intersection-individual,thm:tq-characterization}.
Let \(\tau=f^{-1}\circ f'\) be the transition induced by \Cref{thm:intersection-individual}.
For \(h_i(u;a,a_0)\defeq q_i(u,a)-q_i(u,a_0)\), the transported CMM potential identity implies, for every \(j\ne j'\),
\[
0 =
\sum_{i=1}^k
\partial_1 h_i(\tau_i(c);a,a_0)\partial_{c_jc_{j'}}^2\tau_i(c)
+
\sum_{i=1}^k
\partial_1^2 h_i(\tau_i(c);a,a_0)
\partial_{c_j}\tau_i(c)\partial_{c_{j'}}\tau_i(c).
\]
This is the mixed-derivative system used in the proof of Theorem~1 of \citet{hyvarinen2019nonlinear}.
Consequently, under their Assumption of Variability, \(\tau\) is componentwise up to permutation:
\[
    \tau(c)
    =
    \bigl(\phi_1(c_{\pi(1)}),\ldots,\phi_k(c_{\pi(k)})\bigr).
\]
\end{proposition}

\begin{proof}
By \Cref{thm:tq-characterization}, feature equivalence gives
\[
    \Delta_{a,a_0}^{\bar\bQ'}(c)
    =
    \Delta_{a,a_0}^{\bar\bQ}(\tau(c)).
\]
The candidate potential \(\Delta_{a,a_0}^{\bar\bQ'}(c)\) is coordinate-separable in \(c\) up to an additive constant, so
\[
    \partial_{c_jc_{j'}}^2
    \Delta_{a,a_0}^{\bar\bQ'}(c)
    =
    0
    \qquad
    \text{for every }j\ne j'.
\]
Applying the same mixed derivative to the transported right-hand side gives
\[
    0
    =
    \partial_{c_jc_{j'}}^2
    \left[
        \sum_{i=1}^k h_i(\tau_i(c);a,a_0)
    \right],
\]
because \(\Gamma(a)-\Gamma(a_0)\) is independent of \(c\).
Expanding this derivative by the chain rule gives the displayed mixed-derivative system.
The remaining step is exactly the variability-rank and integration argument in \citet{hyvarinen2019nonlinear}, which forces a monomial Jacobian and hence componentwise recovery up to permutation.
\end{proof}


\subsection{Conditionally exponential families}
\label{appendix:proof-ivae}

\paragraph{CMM translation.}
For the conditionally exponential-family model including iVAE \citep{khemakhem2020variational} and ICE-BeeM \citep{khemakhem2020ice}, the concept density has the form
\[
    p_{\bar\bQ}(c\mid a)
    =
    \frac{Q(c)}{Z(a)}
    \exp\{\inner{f(a)}{T(c)}\},
    \qquad
    X=g(C),
\]
where \(f\colon\cA\to\bbR^m\) is the attribute-dependent natural parameter, \(T\colon\cC\to\bbR^m\) collects sufficient statistics, \(Q\) is a positive base density, and \(Z\) is the normalizer.
The modulator is \(\modval(a)=(f(a),T,Q)\), with indexing kernel \(\bQ=\Push_\eta\) for \(\eta(a)=\modval(a)\).
The shared concept-modulation kernel maps \((\theta,T,Q)\) to the density proportional to \(Q(c)\exp\{\inner{\theta}{T(c)}\}\).
The mixing kernel is \(\bK=\Push_g\), or \(\bK=\widetilde{\bK}\Push_g\) in the noisy-observation version.

\paragraph{Attribute potentials.}
The attribute potential relative to \(a_0\) is
\[
    \Delta_{a,a_0}^{\bar\bQ}(c)
    =
    \inner{f(a)-f(a_0)}{T(c)}
    -
    \{\log Z(a)-\log Z(a_0)\}.
\]
After centering at any \(c_0\), the normalizers cancel:
\[
    \Delta_{a,a_0}^{\bar\bQ}(c)-\Delta_{a,a_0}^{\bar\bQ}(c_0)
    =
    \inner{f(a)-f(a_0)}{T(c)-T(c_0)}.
\]
This centered potential is the affine sufficient-statistic equation used in the iVAE identifiability proof.

\paragraph{Recovery of guarantees.}
\begin{proposition}[CMM recovery of Theorem 1 of \citep{khemakhem2020variational}]
\label{prop:recovery-ivae}
Let \(\M=(\Push_\eta,\bB,\Push_g)\) and \(\M'=(\Push_{\eta'},\bB,\Push_{g'})\) be feature-equivalent conditionally exponential-family CMMs, with \(\eta(a)=(f(a),T,Q)\) and \(\eta'(a)=(f'(a),T',Q')\).
Let \(\tau=g^{-1}\circ g'\) be the transition induced by \Cref{thm:intersection-individual}.
Assume there are attributes \(a_0,a_1,\ldots,a_m\) such that
\[
    L
    \defeq
    \bigl(
        f(a_1)-f(a_0),\ldots,f(a_m)-f(a_0)
    \bigr)
\]
is invertible, and define \(L'\) analogously.
Then the CMM potential identity gives
\[
    T(\tau(c))
    =
    MT'(c)+b,
    \qquad
    M=L^{-\top}(L')^\top,
\]
for some \(b\in\bbR^m\).
This is the \(\sim_A\)-identifiability relation of Theorem~1 of \citet{khemakhem2020variational}.
The further refinement to \(\sim_P\) is the separate rigidity step in their Theorems~2--3.
\end{proposition}

\begin{proof}
By \Cref{thm:tq-characterization}, \(\Delta_{a,a_0}^{\bar\bQ'}(c)=\Delta_{a,a_0}^{\bar\bQ}(\tau(c))\).
Centering at \(c_0\) cancels the normalizers and gives
\[
    \inner{f'(a)-f'(a_0)}{T'(c)-T'(c_0)}
    =
    \inner{f(a)-f(a_0)}{T(\tau(c))-T(\tau(c_0))}.
\]
Stacking this identity over \(a_1,\ldots,a_m\) and solving with \(L\) yields
\[
    T(\tau(c))-T(\tau(c_0))
    =
    L^{-\top}(L')^\top\{T'(c)-T'(c_0)\}.
\]
Absorbing the value at \(c_0\) into \(b\) gives the affine relation.
\end{proof}


\subsection{Linear CRL}
\label{appendix:proof-linear-crl}

\paragraph{CMM translation.}
For the linear CRL model of \citet{squires2023linear}, let \(\cA=\{0\}\cup[K]\), \(\cC=\bbR^d\), and \(\cX=\bbR^p\).
The attribute \(k\in\cA\) indexes an observational or interventional environment, the concept is \(C=Z^{(k)}\), and the feature is \(X^{(k)}=GZ^{(k)}\) with \(G\in\bbR^{p\times d}\) full column rank.
In environment \(k\), the latent variables satisfy
\[
    Z^{(k)}=B_k^{-1}\epsilon,
    \qquad
    \bbE[\epsilon]=0,
    \qquad
    \mathrm{Cov}(\epsilon)=I_d,
\]
where \(B_k=\Omega_k^{-1/2}(I_d-A_k)\) is the normalized structural factor obtained through a perfect intervention: doing such on a node \(i_k\) changes only the \(i_k\)-th row of \(B_0\), replacing it by \(\lambda_ke_{i_k}^\top\).
The indexing kernel is \(\bQ=\Push_\eta\), where \(\eta(k)=B_k\).
The concept-modulation kernel maps \(B\) to the law of \(B^{-1}\epsilon\), and the mixing kernel is deterministic, \(\bK=\Push_G\).

\paragraph{Second-moment proof object.}
Writing \(H=G^\dagger\), the observed pseudoprecision in environment \(k\) is
\[
    \Theta_k
    \defeq
    \mathrm{Cov}(X^{(k)})^\dagger
    =
    H^\top B_k^\top B_kH.
\]
The residual relabeling set is
\[
    S(\cG)
    \defeq
    \setbuild{\sigma\colon[d]\to[d]\ \text{bijective}}{\sigma(j)>\sigma(i)\ \text{for every edge }j\to i\text{ in }\cG}.
\]
Thus the proof object in \citet{squires2023linear} is not an attribute-potential equation but the family of observed pseudoprecision equations.

\paragraph{Recovery of guarantees.}
\begin{proposition}[CMM recovery of Theorem 2 of \citep{squires2023linear}]
\label{prop:recovery-squires-thm2}
Let \(\M=(\Push_\eta,\bB,\Push_G)\) and \(\M'=(\Push_{\widetilde\eta},\bB,\Push_{\widetilde G})\) be feature-equivalent linear CRL CMMs with full-column-rank mixing matrices.
Then the mixing side of \Cref{thm:intersection-individual} forces the induced transition to be linear:
\[
    \widetilde G=GT,
    \qquad
    T=G^\dagger\widetilde G\in\mathrm{GL}(d).
\]
The concept-side equality gives
\[
    \mathrm{Cov}(Z^{(k)})
    =
    T\,\mathrm{Cov}(\widetilde Z^{(k)})\,T^\top,
\]
and hence
\[
    \Theta_k
    =
    H^\top B_k^\top B_kH
    =
    \widetilde H^\top\widetilde B_k^\top\widetilde B_k\widetilde H.
\]
Under Assumptions~1--2 of \citet{squires2023linear} with one perfect intervention per latent node, their Theorem~2 gives identifiability up to \(S(\cG)\).
\end{proposition}

\begin{proof}
From the mixing side relation \(\Push_{\widetilde G}\refmeq\Push_G\Push_\tau\), we have \(\widetilde Gz=G\tau(z)\) for \(\refm\)-a.e. \(z\).
Multiplying by \(G^\dagger\) gives \(\tau(z)=Tz\) with \(T=G^\dagger\widetilde G\), and full column rank makes \(T\) invertible.
The concept-side equality transports covariances by \(T\), which is equivalent to the displayed pseudoprecision equations.
These are exactly the equations analyzed in \citet{squires2023linear}; their generalized-RQ and partial-order argument gives the stated residual ambiguity.
\end{proof}

\begin{remark}
This recovery is intentionally mixing-side rather than potential-ratio based.
Unlike the nonlinear recoveries, the restrictive object in linear CRL is the linear mixing kernel \(\bK=\Push_G\).
The CMM transition is forced to be linear before applying the model-specific rigidity theorem of \citet{squires2023linear}.
\end{remark}


\subsection{Gaussian CRL}
\label{appendix:proof-gaussian-crl}

\paragraph{CMM translation.}
For the Gaussian CRL model of \citet{Buchholz2023learning}, let \(\cA=I\cup\{0\}\), \(\cC=\bbR^d\), and \(\cX=\bbR^{d'}\).
The attribute \(i\in I\) indexes an interventional environment, while \(0\) is the observational environment.
The latent concept is \(C=Z^{(i)}\), and the observed feature is \(X^{(i)}=f(Z^{(i)})\).
In environment \(i\), the structural factor is
\[
    B^{(i)}
    =
    (D^{(i)})^{-1/2}(I_d-A^{(i)}),
\]
and only the row indexed by the target node \(t_i\) differs from the observational factor \(B^{(0)}\).
The scalar \(\eta^{(i)}\) denotes the intervention shift parameter from \citet{Buchholz2023learning}; it is not an indexing map.

Define the environment-specific modulator value by
\[
    \modval_i
    \defeq
    (B^{(i)},\eta^{(i)},t_i),
    \qquad
    i\in I\cup\{0\},
\]
with \(\eta^{(0)}=0\) and arbitrary \(t_0\), since the shift term vanishes in the observational environment.
The deterministic indexing kernel is
\[
    \bQ
    =
    \Push_{i\mapsto \modval_i}.
\]
The shared concept-modulation kernel sends a structural triple \((B,\alpha,t)\) to the Gaussian law induced by \(B^{-1}(\epsilon+\alpha e_t)\):
\[
    \bB(dc\mid B,\alpha,t)
    =
    |\det B|\,(2\pi)^{-d/2}
    \exp\left\{
        -\frac12\norm{Bc-\alpha e_t}^2
    \right\}dc.
\]
Thus, for \(\bar\bQ=\bB\bQ\),
\[
    p_{\bar\bQ}(c\mid i)
    =
    |\det B^{(i)}|\,(2\pi)^{-d/2}
    \exp\left\{
        -\frac12\norm{B^{(i)}c-\eta^{(i)}e_{t_i}}^2
    \right\}.
\]
Equivalently,
\[
    C\mid A=i
    \sim
    Z^{(i)}
    =
    (B^{(i)})^{-1}(\epsilon+\eta^{(i)}e_{t_i}),
    \qquad
    \epsilon\sim\cN(0,I_d).
\]
The mixing kernel is deterministic, \(\bK=\Push_f\).

\paragraph{Attribute potentials.}
Write
\[
    r^{(i)}
    \defeq
    (B^{(i)})^\top e_{t_i},
    \qquad
    \Theta^{(i)}
    \defeq
    (B^{(i)})^\top B^{(i)}.
\]
The attribute potential relative to \(0\) is
\[
    \Delta_{i,0}^{\bar\bQ}(c)
    =
    -\frac12 c^\top(\Theta^{(i)}-\Theta^{(0)})c
    +
    \eta^{(i)}(r^{(i)})^\top c
    +
    \kappa^{(i)},
\]
where
\[
    \kappa^{(i)}
    =
    \log\frac{|\det B^{(i)}|}{|\det B^{(0)}|}
    -
    \frac12(\eta^{(i)})^2
\]
is independent of \(c\).
Equivalently, up to an \(i\)-dependent additive constant, \(-\Delta_{i,0}^{\bar\bQ}\) is the quadratic-linear form
\[
    \frac12 c^\top(\Theta^{(i)}-\Theta^{(0)})c
    -
    \eta^{(i)}(r^{(i)})^\top c.
\]
This is the quadratic transition object used in \citet{Buchholz2023learning}.

\paragraph{Recovery of guarantees.}
\begin{proposition}[CMM recovery of Theorems 1 and 3 of \citep{Buchholz2023learning}]
\label{prop:recovery-buchholz}
Let \(\M\) and \(\M'\) be feature-equivalent Gaussian CRL CMMs.
Write the candidate quantities with tildes, so that the candidate mixing map is \(\widetilde f\), the candidate structural factors are \(\widetilde B^{(i)}\), the candidate shift parameters are \(\widetilde\eta^{(i)}\), and
\[
    \widetilde r^{(i)}
    \defeq
    (\widetilde B^{(i)})^\top e_{\widetilde t_i},
    \qquad
    \widetilde\Theta^{(i)}
    \defeq
    (\widetilde B^{(i)})^\top \widetilde B^{(i)}.
\]
Let
\[
    \tau
    \defeq
    \widetilde f^{-1}\circ f
\]
be the transition from ground-truth latent coordinates to candidate latent coordinates.
Then, for every intervention \(i\in I\), there exists \(b^{(i)}\in\bbR\) such that
\[
    \frac12c^\top(\Theta^{(i)}-\Theta^{(0)})c
    -
    \eta^{(i)}(r^{(i)})^\top c
    =
    \frac12\tau(c)^\top(\widetilde\Theta^{(i)}-\widetilde\Theta^{(0)})\tau(c)
    -
    \widetilde\eta^{(i)}(\widetilde r^{(i)})^\top\tau(c)
    +
    b^{(i)}.
\]
This is the quadratic transition identity used in the proof of Theorem~3 of \citet{Buchholz2023learning}.
Under the assumptions of their Theorem~3, their rigidity argument implies \(\tau(c)=Tc\), equivalently \(\widetilde f=f\circ T^{-1}\).
Under the additional perfect-intervention assumptions of their Theorem~1, their Appendix~B combines this linearity with \citet{squires2023linear} to obtain permutation-and-scaling identifiability.
\end{proposition}

\begin{proof}
Feature equivalence and \Cref{thm:intersection-individual} give a latent transition.
Using the inverse-direction transition \(\tau=\widetilde f^{-1}\circ f\), \Cref{thm:tq-characterization} gives
\[
    \Delta_{i,0}^{\bar\bQ}(c)
    =
    \Delta_{i,0}^{\bar\bQ'}(\tau(c)).
\]
Substituting the Gaussian potential expansions for \(\bar\bQ\) and \(\bar\bQ'\) gives
\[
    -\frac12 c^\top(\Theta^{(i)}-\Theta^{(0)})c
    +
    \eta^{(i)}(r^{(i)})^\top c
    +
    \kappa^{(i)}
    =
    -\frac12 \tau(c)^\top(\widetilde\Theta^{(i)}-\widetilde\Theta^{(0)})\tau(c)
    +
    \widetilde\eta^{(i)}(\widetilde r^{(i)})^\top \tau(c)
    +
    \widetilde\kappa^{(i)}.
\]
Multiplying by \(-1\) and absorbing the constant \(\widetilde\kappa^{(i)}-\kappa^{(i)}\) into \(b^{(i)}\) gives the displayed quadratic identity.
\end{proof}


\subsection{Score-based CRL}
\label{appendix:proof-score-crl}

\paragraph{CMM translation.}
In the setting of \citet{varici2025score}, take \(\cA=\cE\), \(\cC=\bbR^n\), and let the modulator space be the space of admissible local-mechanism values
\[
    \modclass
    =
    \bigsqcup_{G\in\mathrm{DAG}([n])}
    \prod_{i=1}^n\MK(\cC_{\pa_G(i)}\to\cC_i).
\]
Each \(\lambda\in\modclass\) specifies local conditional densities \((p_i^\lambda)_{i=1}^n\) and induces a DAG \(\cG_\lambda\) on \([n]\).
For \(i\in[n]\), write \(\pa_{\cG_\lambda}(i)\) for the parent set induced by \(\lambda\).
Let \(\eta\colon\cE\to\modclass\) be the deterministic indexing map and write \(\lambda_a\defeq\eta(a)\).
The indexing kernel is \(\bQ=\Push_\eta\).
The shared concept-modulation kernel maps a mechanism value to the Markov density over its induced graph:
\[
    \bB(dc\mid\lambda)
    =
    \prod_{i=1}^n
    p_i^\lambda(c_i\mid c_{\pa_{\cG_\lambda}(i)})\,dc.
\]
Thus, for \(\bar\bQ=\bB\Push_\eta\),
\[
    p_{\bar\bQ}(c\mid a)
    =
    \prod_{i=1}^n
    p_i^{\lambda_a}(c_i\mid c_{\pa_{\cG_{\lambda_a}}(i)}).
\]
The mixing kernel is deterministic, \(\bK=\Push_f\).

\paragraph{Attribute potentials.}
For \(a,a'\in\cE\),
\[
    \Delta_{a,a'}^{\bar\bQ}(c)
    =
    \sum_{i=1}^n
    \log p_i^{\lambda_a}(c_i\mid c_{\pa_{\cG_{\lambda_a}}(i)})
    -
    \sum_{i=1}^n
    \log p_i^{\lambda_{a'}}(c_i\mid c_{\pa_{\cG_{\lambda_{a'}}}(i)}).
\]
For a coupled hard-intervention pair \((a_i,\tilde a_i)\) targeting node \(i\), the two mechanism values agree at all non-\(i\) local mechanisms and have parent-free \(i\)-th mechanisms.
Hence
\[
    \Delta_{a_i,\tilde a_i}^{\bar\bQ}(c)
    =
    \log p_i^{\lambda_{a_i}}(c_i)
    -
    \log p_i^{\lambda_{\tilde a_i}}(c_i),
\]
so \(\nabla_c\Delta_{a_i,\tilde a_i}^{\bar\bQ}(c)\) is one-sparse.
The differentiated attribute potential is the latent score-difference object used in score-based CRL.

\paragraph{Recovery of guarantees.}
\begin{proposition}[CMM recovery of Theorem 25 of \citep{varici2025score}]
\label{prop:recovery-varici-thm25}
Let \(\M=(\Push_\eta,\bB,\Push_f)\) and \(\M'=(\Push_{\eta'},\bB,\Push_{f'})\) be feature-equivalent score-based CRL CMMs.
Write \(\lambda_a\defeq\eta(a)\) and \(\lambda'_a\defeq\eta'(a)\), so that the graphs in environment \(a\) are \(\cG_{\lambda_a}\) and \(\cG_{\lambda'_a}\).
Let \(\tau\) be the transition induced by \Cref{thm:intersection-individual}.
For any coupled hard-intervention pair \((a_i,\tilde a_i)\),
\[
    \nabla_c\Delta_{a_i,\tilde a_i}^{\bar\bQ'}(c)
    =
    \partial_c\tau(c)^\top
    \nabla_c\Delta_{a_i,\tilde a_i}^{\bar\bQ}(\tau(c)).
\]
Moreover, the two score contrasts are one-sparse in the corresponding candidate and ground-truth intervention targets.
This is the transported score-contrast and sparsity object used by \citet{varici2025score}; their Theorem~25 gives componentwise recovery and recovery of the mechanism-induced anchor graph \(\cG_{\lambda_{a_0}}\) up to the corresponding permutation of \(\cG_{\lambda'_{a_0}}\).
\end{proposition}

\begin{proof}
By \Cref{thm:tq-characterization}, subtracting the transported identities for \(a_i\) and \(\tilde a_i\) gives \(\Delta_{a_i,\tilde a_i}^{\bar\bQ'}(c)=\Delta_{a_i,\tilde a_i}^{\bar\bQ}(\tau(c))\).
Differentiating gives the displayed score-contrast equation.
For a coupled hard-intervention pair, all non-target mechanisms cancel and the target mechanisms are parent-free, so the ground-truth and candidate score contrasts are one-sparse in their respective targets.
This is exactly the sparsity property formalized in Theorem~7(iii) of \citet{varici2025score}; the componentwise and graph-recovery steps are their Lemmas~23--24 and Theorem~25.
\end{proof}


\subsection{Nonparametric CRL}
\label{appendix:proof-nonparametric-crl}

\paragraph{CMM translation.}
For the nonparametric CRL model of \citet{von2023nonparametric}, let \(\cA=\cE\), \(\cC=\bbR^n\), and \(\cX\subseteq\bbR^d\).
Fix an observational anchor \(e_0\in\cE\).
The concept variable is \(C=(C_1,\ldots,C_n)\), and the observed variable is \(X=f(C)\), where \(f\colon\bbR^n\to\cX\) is a diffeomorphism onto its image.
Let \(\modclass\) be the space of admissible local-mechanism values.
Each \(\lambda\in\modclass\) specifies local conditional densities \((p_j^\lambda)_{j=1}^n\) and induces a DAG \(\cG_\lambda\) on \([n]\).
Let \(\eta\colon\cE\to\modclass\) be the deterministic indexing map, write \(\lambda_e\defeq\eta(e)\), and set \(\lambda_0\defeq\lambda_{e_0}\).
The indexing kernel is \(\bQ=\Push_\eta\).
The shared concept-modulation kernel is
\[
    \bB(dc\mid\lambda)
    =
    \prod_{j=1}^n
    p_j^\lambda(c_j\mid c_{\pa_{\cG_\lambda}(j)})\,dc.
\]
Thus, for \(\bar\bQ=\bB\Push_\eta\),
\[
    p_{\bar\bQ}(c\mid e)
    =
    \prod_{j=1}^n
    p_j^{\lambda_e}(c_j\mid c_{\pa_{\cG_{\lambda_e}}(j)}).
\]
The mixing kernel is deterministic, \(\bK=\Push_f\), and the ground-truth anchor graph is the mechanism-induced graph \(\cG_{\lambda_0}\).

\paragraph{Attribute potentials.}
For \(e,e'\in\cE\), define
\[
    R_{e,e'}^{\bar\bQ}(c)
    \defeq
    \exp\{\Delta_{e,e'}^{\bar\bQ}(c)\}
    =
    \frac{p_{\bar\bQ}(c\mid e)}{p_{\bar\bQ}(c\mid e')}.
\]
For a paired perfect intervention \((e,e')\) on node \(i\), all non-target mechanisms agree and cancel in the ratio.
The target mechanisms are parent-free, so with
\[
    \rho_i^{e,e'}(u)
    \defeq
    \frac{\widetilde p_i^e(u)}{\widetilde p_i^{e'}(u)},
\]
we have
\[
    R_{e,e'}^{\bar\bQ}(c)
    =
    \rho_i^{e,e'}(c_i).
\]
Thus the CMM proof object is the paired environment density ratio used by \citet{von2023nonparametric}.

\paragraph{Recovery of guarantees.}
\begin{proposition}[CMM recovery of Theorem 3.4 of \citep{von2023nonparametric}]
\label{prop:recovery-vonk-thm34}
Let \(\M=(\Push_\eta,\bB,\Push_f)\) and \(\M'=(\Push_{\eta'},\bB,\Push_{h^{-1}})\) be feature-equivalent nonparametric CRL CMMs.
Write \(\lambda_e\defeq\eta(e)\), \(\lambda'_e\defeq\eta'(e)\), \(\lambda_0\defeq\lambda_{e_0}\), and \(\lambda'_0\defeq\lambda'_{e_0}\).
Let \(h\colon\cX\to\bbR^n\) be the candidate unmixing map, and define the transition from candidate to ground-truth coordinates by
\[
    \tau(z)\defeq f^{-1}(h^{-1}(z)).
\]
For any paired perfect intervention \((e,e')\) with ground-truth target \(i\) and candidate target \(j\),
\[
    \rho_j^{\prime e,e'}(z_j)
    =
    \rho_i^{e,e'}(\tau_i(z)).
\]
For the \(n\) paired interventions used in Theorem~3.4 of \citet{von2023nonparametric}, writing \(j=\pi(i)\) gives
\[
    \rho_{\pi(i)}^{\prime e,e'}(z_{\pi(i)})
    =
    \rho_i^{e,e'}(\tau_i(z)).
\]
This is the one-dimensional ratio equation in Appendix~C.3 of \citet{von2023nonparametric}; their theorem gives coordinatewise recovery and graph isomorphism between the mechanism-induced anchor graphs \(\cG_{\lambda_0}\) and \(\cG_{\lambda'_0}\).
\end{proposition}

\begin{proof}
By \Cref{thm:tq-characterization}, feature equivalence gives \(R_{e,e'}^{\bar\bQ'}(z)=R_{e,e'}^{\bar\bQ}(\tau(z))\).
For a paired perfect intervention, all non-target mechanisms cancel in both ratios.
Thus the ground-truth ratio is \(\rho_i^{e,e'}(\tau_i(z))\), while the candidate ratio is \(\rho_j^{\prime e,e'}(z_j)\).
Substitution gives the displayed equation.
The nondegeneracy, coordinatewise-recovery, and graph-isomorphism steps are exactly those in Appendix~C.3 and Theorem~3.4 of \citet{von2023nonparametric}.
\end{proof}


\subsection{Perturbation modeling}
\label{appendix:proof-perturbation-modeling}

\paragraph{CMM translation.}
Perturbation modeling in \citet{von2025representation} fits the CMM template by taking the attribute \(a\in\cA=\bbR^K\) to be a perturbation label, the concept \(c\in\cC=\bbR^k\) to be the perturbation-relevant latent state, and the feature \(x\in\cX\) to be the observed measurement.
Fix an anchor perturbation \(a_0\).
The deterministic indexing kernel is
\[
    \bQ=\Push_{a\mapsto W(a-a_0)}.
\]
The shared concept-modulation kernel is
\[
    \bB(dc\mid \modval)=\cN(\modval,I_k)(dc)
\]
where $\cN(\mu, \Sigma)$ denotes a Gaussian distribution with mean $\mu$ and covariance $\Sigma$.
Thus, for \(\bar\bQ=\bB\bQ\),
\[
    p_{\bar\bQ}(c\mid a)
    =
    \frac{1}{(2\pi)^{k/2}}
    \exp\left\{
        -\frac12\norm{c-W(a-a_0)}^2
    \right\}.
\]
The mixing kernel is deterministic, \(\bK=\Push_g\), where \(g\colon\bbR^k\to\cX\) is the observation map.

\paragraph{Attribute potentials.}
For this Gaussian mean-shift concept law, the attribute potential relative to \(a_0\) is
\[
\begin{aligned}
    \Delta_{a,a_0}^{\bar\bQ}(c)
    &=
    \log p_{\bar\bQ}(c\mid a)-\log p_{\bar\bQ}(c\mid a_0) \\
    &=
    \inner{W(a-a_0)}{c}
    -
    \frac12\norm{W(a-a_0)}^2.
\end{aligned}
\]
Therefore, after centering at any \(c_0\in\bbR^k\),
\[
    \Delta_{a,a_0}^{\bar\bQ}(c)-\Delta_{a,a_0}^{\bar\bQ}(c_0)
    =
    \inner{a-a_0}{W^\top(c-c_0)}.
\]
This is the proof object used in the perturbation model: perturbations act linearly on the latent mean, and centered attribute potentials expose the perturbation effect matrix \(W\).

\paragraph{Recovery of guarantees.}
\begin{proposition}[CMM recovery of perturbation identifiability and extrapolation]
\label{prop:recovery-perturbation-modeling}
Let
\[
    \M=(\Push_{a\mapsto W(a-a_0)},\bB,\Push_g),
    \qquad
    \M'=(\Push_{a\mapsto W'(a-a_0)},\bB,\Push_{g'})
\]
be feature-equivalent perturbation CMMs on observed perturbations \(\cAo\), with injective deterministic mixing maps.
Let \(\tau=g^{-1}\circ g'\) be the transition induced by \Cref{thm:intersection-individual}.
Then, for every \(a\in\cAo\),
\[
    \inner{W'(a-a_0)}{c}
    -
    \frac12\norm{W'(a-a_0)}^2
    =
    \inner{W(a-a_0)}{\tau(c)}
    -
    \frac12\norm{W(a-a_0)}^2.
\]
Equivalently, after centering at \(c_0\),
\[
    \inner{a-a_0}{(W')^\top(c-c_0)}
    =
    \inner{a-a_0}{W^\top(\tau(c)-\tau(c_0))}.
\]
Under the sufficient-diversity condition of \citet{von2025representation}, this is their perturbation-identifiability equation, with the remaining orthogonal-rigidity step supplied by their theorem.
The same centered identity satisfies \Cref{thm:extrapolation-affine-indexing} with \(\varphi(a)=a\), so feature equivalence extrapolates to every \(\aex\) satisfying
\[
    \aex-a_0\in\Span\setbuild{a-a_0}{a\in\cAo}.
\]
\end{proposition}

\begin{proof}
Substituting the Gaussian mean-shift potential into the transported identity \(\Delta_{a,a_0}^{\bar\bQ'}(c)=\Delta_{a,a_0}^{\bar\bQ}(\tau(c))\) gives the first display.
Subtracting the same identity at \(c_0\) gives the centered display.
The identifiability conclusion is the rigidity theorem of \citet{von2025representation}, and the extrapolation conclusion follows from \Cref{thm:extrapolation-affine-indexing,thm:extrapolation-delta-iff}.
\end{proof}


\clearpage
\section{CMM Translations of prior work}\label{appendix:tables}

\begin{table*}[h]
\centering
\scriptsize
\setlength{\tabcolsep}{4pt}
\renewcommand{\arraystretch}{1.15}
\resizebox{\textwidth}{!}{
\begin{tabular}{llcccc}
\toprule\toprule
\textbf{Setting}
& \textbf{Reference}
& \textbf{$A$}
& \textbf{$\modrv$}
& \textbf{$C$}
& \textbf{$X$}
\\
\midrule

\multirow{3}{*}{Nonlinear ICA}
& \citet{hyvarinen2019nonlinear}
& auxiliary variable \(\bu\)
& \(\bigl(q_i(\cdot \mid \bu)\bigr)_{i=1}^k\)
& \(\bs\in\bbR^k\)
& \(\bx = f(\bs)\)
\\

& \citet{khemakhem2020variational}
& conditioning variable \(\bu\)
& \(\bigl(f(\bu),T,Q\bigr)\)
& \(\bz\)
& \(\bx = f(\bz) + \varepsilon\)
\\

& \citet{khemakhem2020ice}
& conditioning variable \(by\)
& \(\bigl(\boldf_\theta, \bg_\theta(\by)\bigr)\)
& \(\bx\)
& \(\bx\)
\\

\midrule

\multirow{4}{*}{CRL}
& \citet{squires2023linear}
& environment \(k\in\{0\}\cup[K]\)
& structural factor \(B_k\)
& \(Z\)
& \(X = GZ\)
\\

& \citet{Buchholz2023learning}
& environment \(i\in I\cup\{0\}\)
& \(\modval_i=(B^{(i)},\eta^{(i)},t_i)\)
& \(Z\)
& \(X = f(Z)\)
\\

& \citet{von2023nonparametric}
& environment \(e\in\cE\)
& local-mechanisms \(\lambda = (p_i^\lambda(c_i \mid c_{\mathrm{pa}_{\cG_\lambda}(i)})_{i=1}^n)\)
& \(\bV\)
& \(\bX\)
\\

& \citet{varici2025score}
& environment \(a\in\cE\)
& local-mechanisms \(\lambda = (p_i^\lambda(c_i \mid c_{\mathrm{pa}_{\cG_\lambda}(i)})_{i=1}^n)\)
& \(\bZ\)
& \(\bX = g(\bZ)\)
\\

\midrule

\multirow{2}{*}{Perturbation modeling}

& \citet{ahuja2022weakly}
& action \(a\)
& weak-supervision mechanism \((S_a,\delta_a)\)
& \(Z\)
& \(X=g(Z)\)
\\

& \citet{von2025representation}
& perturbation label \(a\in\bbR^K\)
& latent mean \(W(a-a_0)\)
& \(C\in\bbR^k\)
& \(X=g(C)\)
\\

\midrule

\multirow{2}{*}{Topic modeling}
& \citet{Arora2012learning}
& document \(d\)
& topic mixture \(w_d\in\Delta^{r-1}\)
& topic \(T\in[r]\)
& word \(V\in[n]\)
\\

& \citet{huang2016anchor}
& document \(d\)
& topic mixture \(w_d\in\Delta^{r-1}\)
& topic \(T\in[r]\)
& word \(V\in[n]\)
\\

\midrule

\multirow{2}{*}{Others}
& \citet{Rajendran2024from}
& atomic concept label \(\alpha=(a,b)\)
& \(\alpha=(a,b)\)
& \(Z\)
& \(X=f(Z)\)
\\

& \citet{SchmidtSchneiderBethge2025equivariance}
& group element \(g\)
& \(\Pull(g)\)
& \((\bx,\bx')\)
& \((\by, \by')\)
\\

\bottomrule\bottomrule
\end{tabular}
}
\captionsetup{skip=3pt}
\caption{Variable-level CMM translations for representative prior work. Each row records the objects along the chain \(A\to\modrv\to C\to X\), following notations in the original papers.
In the mechanism-based CRL rows, \(\cG_{\lambda}\) denotes the induced DAG from local mechanisms $\lambda$.}
\label{tab:related_work_objects_full}
\end{table*}

\begin{table*}[ht]
\centering
\scriptsize
\setlength{\tabcolsep}{4pt}
\renewcommand{\arraystretch}{1.15}
\resizebox{\textwidth}{!}{
\begin{tabular}{llccc}
\toprule\toprule
\textbf{Setting}
& \textbf{Reference}
& \textbf{\(\bQ\in\MK(\cA\to\modclass)\)}
& \textbf{\(\bB\in\MK(\modclass\to\cC)\)}
& \textbf{\(\bK\in\MK(\cC\to\cX)\)}
\\
\midrule

\multirow{3}{*}{Nonlinear ICA}
& \citet{hyvarinen2019nonlinear}
& \(\Push_{a\mapsto(q_i(\cdot,a))_{i=1}^k}\)
& \(\displaystyle \bB(dc\mid\modval)=\exp\left\{\sum_{i=1}^k\modval_i(c_i)-\Gamma(\modval)\right\}dc\)
& \(\Push_f\)
\\

& \citet{khemakhem2020variational}
& \(\Push_{a\mapsto(f(a),T,Q)}\)
& \(\displaystyle \bB(dc\mid\theta,T,Q)=\frac{Q(c)\exp\{\inner{\theta}{T(c)}\}}{Z(\theta)}\,dc\)
& \(\widetilde{\bK}\Push_g\)
\\

& \citet{khemakhem2020ice}
& \(\Push_{a\mapsto(\bg_\theta(a),\boldf_\theta)}\)
& \(\displaystyle \bB(dc\mid\gamma,\boldf)=\frac{\exp\{-\inner{\gamma}{\boldf(c)}\}}{Z(\gamma,\boldf)}\,dc\)
& \(\Push_{\id}\)
\\

\midrule

\multirow{4}{*}{CRL}
& \citet{squires2023linear}
& \(\displaystyle \Push_{k\mapsto B_k}\)
& \(\displaystyle \bB(dc\mid B)=\mathrm{Law}(B^{-1}\epsilon)(dc)\)
& \(\displaystyle \Push_G\)
\\

& \citet{Buchholz2023learning}
& \(\displaystyle \Push_{i\mapsto(B^{(i)},\eta^{(i)},t_i)}\)
& \(\displaystyle \bB(dc\mid B,\eta,t)=|\det B|(2\pi)^{-d/2}\exp\left\{-\frac12\norm{Bc-\eta e_t}^2\right\}dc\)
& \(\displaystyle \Push_f\)
\\

& \citet{von2023nonparametric}
& \(\displaystyle \Push_{e\mapsto\lambda_e}\)
& \(\displaystyle \bB(dc\mid\lambda)=\prod_{j=1}^n p_j^\lambda(c_j\mid c_{\pa_{\cG_\lambda}(j)})\,dc\)
& \(\displaystyle \Push_f\)
\\

& \citet{varici2025score}
& \(\displaystyle \Push_{a\mapsto\lambda_a}\)
& \(\displaystyle \bB(dc\mid\lambda)=\prod_{i=1}^n p_i^\lambda(c_i\mid c_{\pa_{\cG_\lambda}(i)})\,dc\)
& \(\displaystyle \Push_f\)
\\

\midrule

\multirow{2}{*}{Perturbation modeling}
& \citet{ahuja2022weakly}
& \(\displaystyle \Push_{a\mapsto(S_a,\delta_a)}\)
& \(\displaystyle \bB(dc\mid S,\delta)=\mathrm{Law}(\text{latent transformed by }(S,\delta))(dc)\)
& \(\displaystyle \Push_g\)
\\

& \citet{von2025representation}
& \(\displaystyle \Push_{a\mapsto W(a-a_0)}\)
& \(\displaystyle \bB(dc\mid \lambda)=\cN(\lambda,I_k)(dc)\)
& \(\displaystyle \Push_g\)
\\

\midrule

\multirow{2}{*}{Topic modeling}
& \citet{Arora2012learning}
& \(\displaystyle \Push_{d\mapsto w_d}\)
& \(\displaystyle \bB(T=t\mid w)=w_t\)
& \(\displaystyle \bK(V=v\mid T=t)=\Phi_{tv}\)
\\

& \citet{huang2016anchor}
& \(\displaystyle \Push_{d\mapsto w_d}\)
& \(\displaystyle \bB(T=t\mid w)=w_t\)
& \(\displaystyle \bK(V=v\mid T=t)=\Phi_{tv}\)
\\

\midrule

\multirow{2}{*}{Others}
& \citet{Rajendran2024from}
& \(\displaystyle \Push_{(a,b)\mapsto(a,b)}\)
& \(\displaystyle \bB(dc\mid a,b)=\frac{p(c)q(\inner{a}{z}-b)}{Z(a,b)}\,dc\)
& \(\displaystyle \Push_f\)
\\

& \citet{SchmidtSchneiderBethge2025equivariance}
& \(\displaystyle \Push_{g\mapsto\Pull(g)}\)
& \(\displaystyle \bB(d(u,u')\mid\rho)=\mathrm{Law}(U,\rho U)(d(u,u'))\)
& \(\displaystyle \Push_{f^{\otimes 2}}\)
\\

\bottomrule\bottomrule
\end{tabular}
}
\captionsetup{skip=3pt}
\caption{Operator-level CMM translations for representative prior work. Each row records the three kernels in the CMM chain \(A\to\modrv\to C\to X\). 
For \citet{hyvarinen2019nonlinear}, \(\Gamma(\modval)\) is the normalizer determined by the tuple \(\modval=(q_i)_{i=1}^k\), not an additional modulator component.
For \citep{khemakhem2020variational}, $\widetilde{\bK}$ denotes a kernel that adds noise, which are injective through deconvolution process.
In the Gaussian CRL row, \(\eta^{(i)}\) is the shift parameter used in that model, not an indexing map. In the mechanism-based CRL rows, \(\lambda\) denotes a tuple of local mechanisms and \(\cG_\lambda\) denotes the DAG induced by those mechanisms. For topic models, \(\Phi\in\bbR^{r\times n}\) denotes the topic-word matrix, so \(\bK(V=v\mid T=t)=\Phi_{tv}\).}
\label{tab:related_work_operators_full}
\end{table*}

\end{document}